\pdfoutput=1
\documentclass[letterpaper]{article}
\usepackage[totalwidth=480pt, totalheight=680pt]{geometry}
\usepackage{macro}

\usepackage{authblk}

\title{On the Number of Conditional Independence Tests\\in
Constraint-based Causal Discovery}

\author[1,2,$\dagger$,\footnote{
        M. Franquesa Mon\'{e}s contributed to this work while
        visiting the Institute for Data, Systems, and Society at MIT.
}]{Marc Franquesa Mon\'{e}s}
\author[1,3,$\dagger$]{Jiaqi Zhang}
\author[1,3]{Caroline Uhler}
\affil[1]{Eric and Wendy Schmidt Center, Broad Institute of MIT and Harvard}
\affil[2]{Centre de Formaci\'{o} Interdisciplin\`{a}ria Superior,
Universitat Polit\`{e}cnica de Catalunya}
\affil[3]{Laboratory for Information and Decision Systems,
Massachusetts Institute of Technology}
\affil[$\dagger$]{Equal Contribution}

\date{}

\begin{document}

\maketitle

\begin{abstract}
    Learning causal relations from observational data is a
fundamental problem with wide-ranging applications across many fields.
Constraint-based methods infer the underlying causal
structure by performing conditional independence tests. 
However, existing algorithms such as the prominent \texttt{PC} algorithm need to perform
a large number of independence tests, which in the worst case is
exponential in the maximum degree of the causal graph.
Despite extensive research,
it remains unclear if there exist algorithms with better
complexity without additional assumptions.
Here, we establish an algorithm that achieves a better complexity of $p^{\mathcal{O}(s)}$ tests, where $p$ is the number of nodes in the graph and $s$ denotes the maximum undirected clique size of the underlying essential graph.
Complementing this result, we prove that any constraint-based algorithm
must perform at least $2^{\Omega(s)}$ conditional
independence tests, establishing that our proposed algorithm
achieves exponent-optimality up to a logarithmic factor in terms of the number of conditional
independence tests needed.
Finally, we validate our theoretical findings through
simulations, on semi-synthetic gene-expression data, and real-world
data, demonstrating the efficiency of our algorithm compared to
existing methods in terms of number of conditional independence tests needed.
 \end{abstract}

\section{Introduction}\label{sec:intro}

Inferring causal relationships between variables is a fundamental
challenge across many scientific domains
\citep{wright1921correlation,haavelmo1943statistical,duncan1975introduction}.
Directed acyclic graphs (DAGs) are a popular framework for
representing causal structures, where nodes correspond to
random variables and directed edges denote direct causal effects.
The objective of causal structure discovery is to identify these edges and
their orientations from data on the nodes.
With observational data and no additional assumptions, it is known
that the underlying causal DAG is only identifiable up to its Markov
equivalence class~\citep{verma1990}, which can be represented by
a partially directed graph known as the \emph{essential graph}.

Constraint-based causal structure discovery methods use conditional
independence (CI) tests to infer the underlying essential graph, the
most prominent example being the \texttt{PC} algorithm \citep{spirtes2001}: (1)
it begins with a fully connected undirected graph and iteratively
removes edges between variables if they are found to be conditionally
independent given some conditioning set; (2) edge orientations are
inferred based on cases where variables that are conditionally
independent become dependent when conditioned on a common
child—corresponding to v-structures in the underlying DAG; (3)
possibly additional orientations are inferred using a set of logical
rules known as the Meek rules~\citep{meek1995}.
Building on this approach, many constraint-based algorithms
have been introduced to address more general and diverse settings,
e.g., including in the presence of latent
variables~\citep{verma1990,spirtes1989causality,spirtes2001,colombo2012learning,colombo2014order,sondhi2019reduced}.

\ifbool{aistats}{
    {\let\thefootnote\relax\footnotetext{$^*$Equal contribution. M.F.M. contributed to this work while
        visiting the Institute for Data, Systems, and Society at MIT.}}
}{}

Despite significant progress in constraint-based algorithms, it has
been observed that their efficiency, both in terms of computational
speed and correctness, tend to degrade when scaling to larger or
denser graphs~\citep{spirtes2001,kalisch2007estimating}.
The main bottleneck of the \texttt{PC} algorithm (and other
algorithms of similar nature) arises from the adjacency search
(step (1) described above), which requires a large number of CI
tests, exponential in the maximum degree of the underlying DAG in the
worst case~\citep{spirtes2001,claassen2013learning}.
Performing a large number of CI tests not only exacerbates the
computational burden but may also impact accuracy, since in the
finite sample regime guaranteeing the correctness of each CI test
requires strong
assumptions~\citep{uhler2013geometry,teh2024general,mazaheri2025faithfulness}.
This motivates the development of causal structure discovery
algorithms that minimize the number of CI tests that need to be performed.

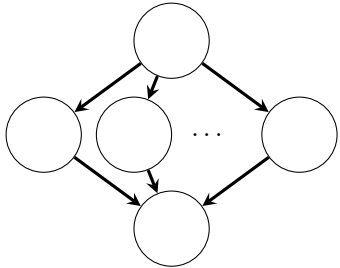
\begin{figure}[htpb]
    \centering
    \scalebox{\ifbool{aistats}{0.5}{1}}{\begin{tikzpicture}
    \node[base] (0) at (0,0) {};

    \node[base] (1) at (0,-2.5) {};

    \node[base] (2) at (-1.7,-1.25) {};
    \node[base] (3) at (-0.5,-1.25) {};
    \node[] (dots) at (0.5,-1.25) {\dots};
    \node[base] (n) at (1.7,-1.25) {};

    \draw[dedge] (0) -- (2);
    \draw[dedge] (2) -- (1);

    \draw[dedge] (0) -- (3);
    \draw[dedge] (3) -- (1);

    \draw[dedge] (0) -- (n);
    \draw[dedge] (n) -- (1);
\end{tikzpicture}
     }
    \caption{
        A graph with $p$ nodes, where $d=p-2$, and $s=2$.
    }
    \label{fig:sep-graph}
\end{figure}

In this work, we present an algorithm, called \emph{\algofullname}
(\texttt{\algoname}), that is exponent-optimal up to a logarithmic
factor in terms
of the number of CI tests performed.
\texttt{\algoname} requires at most $p^{\mathcal{O}(s)}$ CI tests,\footnote{Here, we
    use the asymptotic notation $\mathcal{O}$
    (c.f.,~\citep{arora2009computational}) in the following way: a
    function $f:\mathcal{G}\mapsto \mathbb{R}$ is $p^{\mathcal{O}(s)}$ if
    and only if there exists a fixed constant $c\in\mathbb{R}$, such that
    for any DAG $\mathcal{G}$, it holds that $f(\mathcal{G})\leq
    p^{c\cdot s}$, where $p$ and $s$ are the number of nodes and the
    maximum undirected clique size of $\mathcal{G}$, respectively. We
define $2^{\Omega(s)}$ and $p^{\mathcal{O}(d)}$ in a similar way.}
where $p$ is the number of nodes in the DAG and $s$
denotes the size of the maximum undirected clique in the underlying
essential graph. The best upper bound for previous constraint-based
algorithms~\citep{richardson1996fci,spirtes2001,claassen2013learning,colombo2014order}
was $p^{\mathcal{O}{(d)}}$, where $d$ is the graph's maximum degree. Considering the degree of a node in the
maximum undirected clique of the essential graph, note that $s\leq
d+1$. In fact, $s$ is often much smaller than $d$; see, for example,
the DAG in Figure~\ref{fig:sep-graph}; thus, $p^{\mathcal{O}(s)}$ is
often much smaller than $p^{\mathcal{O}{(d)}}$. Complementing this
upper bound, we prove that any constraint-based
algorithm must perform $2^{\Omega(s)}$ CI tests utilizing a
technique introduced in \citet{zhang2024membership}, thereby
establishing that \texttt{GAS} achieves exponent-optimality up to a
logarithmic factor in terms of the number of CI tests needed. Since
$d$ can be of the size $\Omega(p\cdot s)$ (e.g., as in the graph in
Figure~\ref{fig:sep-graph}), the previous constraint-based algorithms
do not achieve exponent-optimality up to a logarithmic factor.
Finally, we empirically benchmark our algorithm against established
methods on both synthetic and real-world data, assessing both
computational efficiency and accuracy.

To provide some insight on how the reduction in number of CI tests is
achieved, consider the \texttt{PC} algorithm.
The adjacency search in step (1) of the \texttt{PC} algorithm is where the CI
tests are performed.
To reduce the number of CI tests compared to the \texttt{PC} algorithm, \texttt{GAS}
integrates steps (1) and (2) of the \texttt{PC} algorithm; namely, \texttt{GAS} focuses
on using CI tests to learn \emph{ancestral relationships}, which can
then be used to perform CI tests to uncover \emph{adjacencies} in a
more targeted way.
In particular, note that if all ancestral relationships were known,
then a single CI test would be sufficient to determine the
presence/absence of an edge.\footnote{All ancestral relationships
    being known is equivalent to being given a permutation
    $\bm{\pi}=(\pi_1,\dots,\pi_p)$ of the nodes that is \emph{consistent}
    with the DAG (i.e., if there is a directed edge $\pi_i\to\pi_j\in E$,
    then it has to hold that  $i<j$). In this case a single CI test is
    sufficient to determine the presence/absence of an edge, namely
    $\pi_i \to\pi_j$ for $i<j$ if and only if $X_{\pi_i}$ is not
    independent of $X_{\pi_j}$ given $X_S$ where $S=\{\pi_1, \pi_2, \dots
, \pi_{j-1}\}\setminus\{\pi_i\}$.}
As a consequence, the main difficulty of causal structure discovery
is to learn the ancestral relations, which motivates our approach of
concentrating on CI tests that provide ancestral information.
Recent work~\citep{shiragur2024causal} has shown that specific
ancestral relationships can be learned using a polynomial number of CI tests.
We show that, with appropriate modifications, this approach can be
extended to learn all the adjacencies with substantially fewer CI tests
than other methods.

\subsection{Related works}\label{sec:related-work}

Learning causal structures from observational data has been
extensively studied, with current methods primarily falling into
three categories \citep{kitson2023survey}: constraint-based
approaches \citep{verma1990, spirtes2001,kalisch2007estimating},
which utilize CI tests as constraints to learn the DAG; score-based approaches
\citep{chickering2002optimal,geiger2002parameter,brenner2013sparsityboost},
which search for DAGs that maximize a score function evaluating how
well the DAG represents the data; and hybrid approaches
\citep{schulte2010imap,alonso2013scaling,nandy2018high,solus2021consistency},
which combine both.

For constraint-based algorithms, the focus of this work, a major
challenge is their CI test complexity.
Algorithms such as \texttt{PC} have a CI test complexity that is exponential
in the graph's maximum degree~\citep{spirtes2001,claassen2013learning}.
While improvements exist that depend on other parameters like the
maximum in-degree \citep{mokhtarian2025recursive} or the maximum
separating set size \citep{sondhi2019reduced}, they face similar scalability
issues or require additional assumptions.

To address this complexity, several research directions have emerged.
Some methods leverage minimal conditional (in)dependencies to obtain
ancestral structures \citep{claassen2013learning, magliacane2016ancestral}.
Other recent work shifts the focus from learning the entire graph to
characterizing which parts of it can be learned under a constrained
number of CI tests \citep{shiragur2024causal} or with small
conditioning sets \citep{kocaoglu2023characterization}. Alternative
formulations, such as testing a hypothesized DAG, have also been
explored~\citep{zhang2024membership}.
While valuable, these approaches either change the problem's scope or
do not aim to recover the complete essential graph.
In this work, by contrast, we recover the full essential graph with
fewer number of CI tests.

While discrete-variable Bayesian network learning is known to be
NP-hard \citep{chickering2004large}, causal discovery with CI
oracles constitutes a distinct problem class that is not NP-hard
under bounded degree assumptions \citep{claassen2013learning}.
Under the causal sufficiency assumption, our derived bounds guarantee that the essential graph can be
recovered with a polynomial number of tests provided the size of
the undirected cliques is bounded.
This condition covers the case of bounded degree.
Furthermore, for the general case of unbounded clique sizes, our
lower bound suggests that constraint-based discovery with CI
oracles is not in NP, since verifying a specific Markov
equivalence class requires an exponential number of queries, but
remains within EXP as the number of all possible DAGs is bounded
by $2^{poly(p)}$.

Beyond computational efficiency, the correctness of any causal
discovery algorithm depends on statistical assumptions that connect
the observed data to the underlying DAG.
For classic algorithms like \texttt{PC},
correctness in the
infinite-sample limit is guaranteed by the Markov and faithfulness
assumptions~\citep{lauritzen1996,spirtes2001}.
Recent work provides a framework for characterizing the precise
correctness conditions for any given constraint-based algorithm,
enabling a formal comparison of the required
assumptions~\citep{sadeghi2017faithfulness,teh2024general}.
A related line of work explores how additional or redundant CI tests
can be used to detect and correct errors, thereby increasing
robustness~\citep{faller2025different}.
These approaches highlight a fundamental goal: to precisely
characterize the necessary and sufficient CI information required to
recover the correct causal structure.
In this work, we contribute to both fronts: we determine the
number of CI tests required to recover the essential graph, and we
show that our method requires strictly weaker correctness assumptions
than the standard combination of Markov and faithfulness.
Our empirical results further show that our algorithm achieves higher
accuracy than existing methods in denser graphs as well as on synthetic gene
expression data.

\subsection{Organization}

We begin by reviewing essential background and notation in
Section~\ref{sec:preliminaries}.
Our main theoretical results are summarized in Section~\ref{sec:main-results}.
The proposed algorithm, along with its correctness guarantees and
complexity, is presented in Section~\ref{sec:upper-bound}. In
Section~\ref{sec:lower-bound}, we establish a lower bound in terms of
number of CI tests, proving the optimality of our algorithm.
Empirical validation  is provided in Section~\ref{sec:experiments}.
We conclude with a brief discussion in Section~\ref{sec:discussion}.
 \section{Preliminaries}\label{sec:preliminaries}

\subsection{Graph definitions}

Let $ \mathcal{G}$ be a graph consisting of a finite set of $p$
nodes $ V $ and a set of edges $ E $.
For any two nodes $ u, v \in V $, we write $ u \sim v $ if they
are adjacent and $ u \not\sim v $ otherwise.
Edges can be undirected $ u - v $ or directed $ u \rightarrow v $, $
u \leftarrow v $.
For any subset of nodes $W \subseteq V$, we denote its complement
by ${\bar{W} = V \setminus W}$ and the subgraph it induces by $G[W]$.
A \emph{path} in $ \mathcal{G} $ is a sequence of nodes such that
consecutive nodes are adjacent.
A path that starts and ends at the same node is a \emph{cycle}.
A cycle is \emph{directed} if all its edges are directed and follow a
consistent orientation.
A \emph{clique} is a subset of nodes where every two nodes
are adjacent.
In an \emph{undirected clique} the nodes are connected by
undirected edges.
A graph containing directed and/or undirected edges is \emph{partially directed}.
A partially directed graph is a \emph{chain graph} if it has no
directed cycle.
The connected components of a chain graph formed by removing
all directed edges are known as \emph{chain components}.
An undirected graph is \textit{chordal} if every cycle of length four
or more has a chord, that is, an edge connecting two non-consecutive
nodes in the cycle.
A fully directed chain graph is called a \emph{directed acyclic graph} (DAG).
The \emph{skeleton} of a graph, denoted by $
\text{sk}(\mathcal{G}) $, is the undirected graph that results when
all directed edges are replaced by undirected edges.

In directed graphs, for any node $ v \in V $ we denote its sets of
\emph{parents}, \emph{ancestors}, \emph{children}, and
\emph{descendants} in $ \mathcal{G} $ by $ \Pa(v) $, $ \Anc(v) $, $
\Ch(v) $ and $ \Des(v) $, respectively.
We use square brackets for inclusive sets: e.g., $ \Anc[v] = \Anc(v) \cup
\{v\} $.
Following \citet{shiragur2024causal}, a subset $ S \subseteq V $ is
called a \emph{prefix node set} if for all $ v \in S $, $ \Anc[v]
\subseteq S $; in other words, a prefix node set is a subset of
nodes that is closed under ancestral relations.
For any subset $ W \subseteq V $, we define the set of \emph{source
nodes} within $ W $ as $ {\src(W) = \{v \in W \mid \Anc(v) \cap W =
\varnothing\} }$.
When necessary to distinguish the underlying graph, we employ
subscript notation, for example, $ \Pa_{\mathcal{G}} $, $ \src_{\mathcal{G}} $.
A node $ v $ is called a \emph{collider} on a path if its adjacent
nodes $ u $ and $ w $ on the path satisfies $ u \rightarrow v \leftarrow w $.
If these parent nodes $ u $ and $ w $ are not adjacent, this configuration is called a \emph{v-structure}.

\subsection{Markov equivalence classes}\label{sec:2.2}

DAGs are commonly used in causality
(\citealp{lauritzen1982,pearl1985,pearl2009causality}) where nodes
represent random variables and edges represent direct causal relationships.
A distribution $ \mathbb{P} $ is \emph{Markov}
relative to a DAG $ \mathcal{G} $, if it
factorizes as $ {\mathbb{P}(v_1,\dots,v_p)} =
{\prod_{i = 1}^{p} \mathbb{P}(v_i \mid \Pa(v_i)) }$.
This factorization implies a set of conditional independence
relations, which can be determined from $ \mathcal{G} $ using
\emph{d-separation} (\citealp{pearl1988}).
For disjoint node sets $ {A, B, C \subset V }$, $ A $
and $ B $ are \emph{d-separated} by $ C $ if all paths connecting $ A
$ and $ B $ are \emph{inactive} given $ C $.
A path is \emph{inactive} (or \emph{blocked}) given $ C $
if it contains at least one node $v$ satisfying one of the following conditions: (i) $
v $ is a non-collider on the path and $ v \in C $, or (ii) $ v $ is a
collider on the path and $ \Des[v] \cap C = \varnothing $.
Otherwise, the path is \emph{active} given $ C $.
We denote conditional independence in the distribution $ \mathbb{P} $
by $ A \CI B \mid C $, and d-separation in the DAG $ \mathcal{G} $ by
$ A \CI_{\mathcal{G}} B \mid C $.
For singleton sets, for example $ A = \{a\}$, we omit the braces for
simplicity, $ {a \CI B \mid C} $.
Additionally, we write $ A \CI B \mid C $ for potentially
overlapping sets to denote $ A \CI B \mid C \setminus (A \cup B) $.

Multiple DAGs can represent the same conditional independence
information (\citealp{verma1990}).
The set of all DAGs encoding the same independencies forms a
\emph{Markov equivalence class} (MEC), denoted $ [\mathcal{G}] $.
Markov equivalent DAGs are characterized by having the same skeleton
and the same v-structures, which enables a unique representation of an MEC by its
\emph{essential graph}, $ \mathcal{E}(\mathcal{G}) $ (\citealp{andersson1997}).
The essential graph $ \mathcal{E}(\mathcal{G}) $ is a chain graph that
shares the skeleton of the DAGs in $ [\mathcal{G}] $ and
contains a directed edge if and only if the orientation of the edge is the same in
every DAG belonging to the MEC.
These invariant edge orientations stem from v-structures and
orientations given by the \emph{Meek rules} \citep{meek1995,andersson1997}.

If a distribution $ \mathbb{P} $ is Markov with respect to a DAG $
\mathcal{G}$, then d-separation implies conditional independence; i.e., $ A
\CI_{\mathcal{G}} B \mid C \implies A \CI B \mid C $.
Under the \textit{faithfulness} assumption, the reverse implication holds.
When $ \mathbb{P} $ is both Markov and faithful relative to $
\mathcal{G} $, we say that $ \mathbb{P} $ \emph{respects} $ \mathcal{G} $,
signifying: $ A \CI B \mid C \iff A \CI_{\mathcal{G}} B \mid C $ .

 \section{Main results}\label{sec:main-results}

Under the Markov and faithfulness assumptions,\footnote{We remark
    here that we do not need
    to assume causal sufficiency for Theorem~\ref{thm:algorithm} to hold.
    Causal sufficiency needs to be assumed when one requires that
    $\mathcal{G}$ captures the complete causal explanation
\citep{meek1997graphical}.} we provide an algorithm, called
\emph{\algofullname} (\texttt{\algoname}), that achieves
optimality in terms of the number of CI tests. Our algorithm is given
in Section~\ref{sec:algorithm}, with the following guarantees.

\begin{restatable}{theorem}{thmAlgorithm}\label{thm:algorithm}
    Given infinite samples from a distribution respecting a DAG $
    \mathcal{G} $, \texttt{\algoname} (i.e.,
    Algorithm~\ref{alg:learning}) outputs $ \mathcal{E}(\mathcal{G})
    $ using $p^{\mathcal{O}(s)}$ CI tests.
    Here, $ s $ is the size of the maximum undirected clique in $
    \mathcal{E}(\mathcal{G}) $.
\end{restatable}

This result shows that the number of CI tests required by GAS is
upper bounded exponentially in the maximum undirected clique size of the
essential graph. This is in contrast to current algorithms, which,
without additional
assumptions, require in the worst case the number of CI tests to be at least an
exponential function of the maximum
degree~\citep{spirtes2001,claassen2013learning}.
Complementing this result, we show that any algorithm requires at least
$2^{\Omega(s)}$ CI tests to distinguish the ground-truth
$[\mathcal{G}]$ versus other MECs.

\begin{restatable}{theorem}{thmLower}\label{thm:lower}
    Given infinite samples from a distribution respecting a DAG $
    \mathcal{G} $, any algorithm requires at least $
    2^{\Omega(s)}$ CI tests to verify $ \mathcal{G} \in
    [\mathcal{G}] $. Specifically, for any collection of fewer than
    $2^s-s-1$ CI tests, there exists an MEC
    $[\mathcal{H}]\neq[\mathcal{G}]$ such that both classes are
    indistinguishable based on these CI relations.
\end{restatable}

Since any constraint-based causal discovery algorithm must
distinguish the true MEC from all others, it follows that, without
additional assumptions, it must perform at least $2^{\Omega(s)}$ CI
tests. Together, Theorems~\ref{thm:algorithm} and~\ref{thm:lower}
show that the complexity of constraint-based causal discovery, in
terms of the number of CI tests, is exponential in $s$, and that our
algorithm achieves this bound.

Additionally, we show that our algorithm can output the correct MEC
when faithfulness is violated. As a consequence, its correctness
condition is weaker than
Markov and faithfulness, and hence it can output the correct
essential graph in more settings. There exist
other algorithms that also provably require weaker assumptions than Markov and
faithfulness
\citep{ramsey2006adjacency,raskutti2018learning,solus2021consistency},
their complexity in terms of CI testing could be prohibitive.
The proof of the following result can be found in
Appendix~\ref{app:consistency}.

\begin{proposition}\label{prop:algo-assumptions}
    There exists a joint distribution that is Markov but not faithful
    to a DAG $\mathcal{G}$, for which Algorithm~\ref{alg:learning},
    when given observational data sampled from it, correctly outputs
    the ground-truth $\mathcal{E}(\mathcal{G})$.
\end{proposition}

In the remainder of this section, we provide a concise overview of
the methods used to establish our main theorems.

\subsection{Outline of methods}

\textbf{Upper bound for \texttt{\algoname}.}
Compared to other constraint-based algorithm (e.g., \texttt{PC}), the
reason our algorithm performs fewer CI tests is its use of
ancestral information during the adjacency search, which is learned
by keeping track
of both conditional independence \emph{and} dependence statements.
This ancestral information is inferred through two types of CI tests
(given in Definitions~\ref{def:vstructset} and~\ref{def:meekset} below), which
primarily serve to identify edge orientations arising from
v-structures and Meek rule 1. As we show in
Section~\ref{sec:algorithm}, these edge orientations are sufficient to recover
all ancestral information in the underlying essential graph.
\texttt{\algoname} tracks
this ancestral information through an iterative procedure that expands a
prefix node set $S\subset V$.
This process, which builds on~\citet{shiragur2024causal}, serves two key
roles: it enables learning the remaining ancestral
relations in the
essential graph, and it simplifies CI testing for adjacencies by
restricting the search space of the conditioning sets.

Next, to show that the number of CI tests is bounded by $s$, we prove that
the iterative expansion procedure terminates with a
partition of the nodes $V$ into disjoint chordal components $S_1,
S_2, \dots, S_m$, where each partial union $ S_1 \cup S_2 \cup \dots
\cup S_i $ for $i=1,\dots,m$ forms a prefix node set.
It follows from \citet{dirac1961} that any minimal
separator within a chordal component $S_i$ forms a clique. This
property implies that the conditioning sets used for CI tests within
$S_i$ need not exceed the size of the largest clique in the skeleton,
which can serve as a stopping criterion when searching for
adjacencies within $S_i$. Further details are provided in
Section~\ref{sec:upper-bound}.

\textbf{Lower bound on any constraint-based algorithm.}
The argument is based on the largest undirected clique of size $s$ in
the essential graph.
This clique contains $ 2^s - s - 1 $ distinct subsets of size at most
$ s - 2 $, each representing a potential conditioning set.
If an algorithm performs fewer CI tests, there must exist at least
one subset that has not been used as the conditioning set.
For any such unused set, we show how to construct a distinct MEC that
is consistent with all performed CI tests.
A formal proof detailing this construction is provided in
Section~\ref{sec:lower-bound}.

 \section{Upper bound}\label{sec:upper-bound}

In this section, we present our main algorithm and show that it
outputs the correct essential graph using at most
$p^{\mathcal{O}(s)}$ CI tests.
Proofs omitted from this section are
provided in Appendix \ref{app:upper-bound}.

\subsection{Learning a prefix node set}\label{sec:prefix-vertex-set}

From the orientation rules of \texttt{PC} (steps (2) and (3)
described in Section~\ref{sec:intro}), we see that the directional
information  in the essential graph are derived from v-structures and
Meek rules.
While four Meek rules exist for orienting edges, only rule 1
introduces ancestral information not already captured by transitivity
(see Figure~\ref{fig:meek-set-rule}).
Therefore, the ancestral relationships in an essential graph can be
completely determined by its v-structures and the application of Meek rule 1.
Leveraging this observation, we developed two sets of CI tests to
specifically identify the ancestral relations established by
v-structures (Definition~\ref{def:vstructset}) and Meek rule 1
(Definition~\ref{def:meekset}).

\ifbool{aistats}{}{
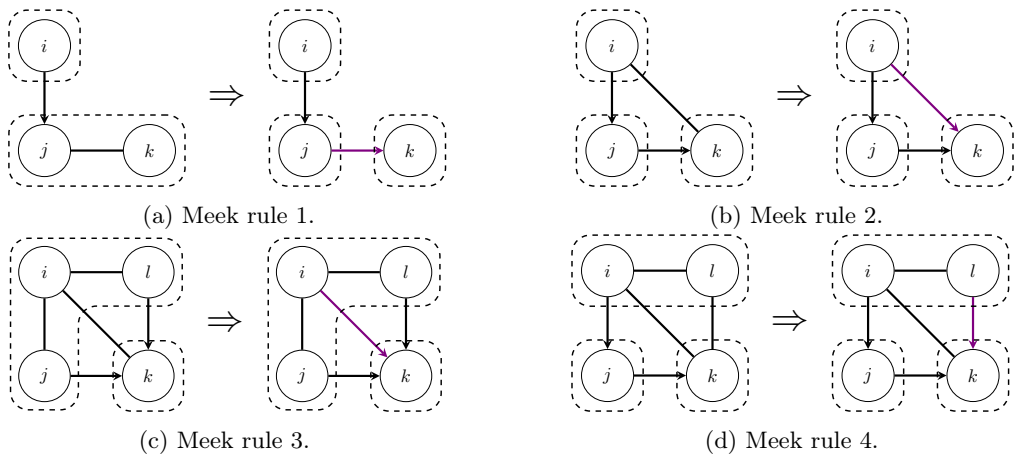
\begin{figure}[h!]
    \def\size{0.35}
    \def\sep{4em}
    \centering
    \begin{subfigure}[b]{\size\textwidth}
        \centering
        \resizebox{\linewidth}{!}{\begin{tikzpicture}
    \begin{scope}
        \node[base] (i) at (0, 0) {$i$};
        \node[base] (j) [below = of i] {$j$};
        \node[base] (k) [right =  of j] {$k$};

        \draw[dedge] (i) to (j);
        \draw[uedge] (j) to (k);

        \node[fitbox, fit=(i)] (0) {};
        \node[fitbox, fit=(j) (k)] (0) {};
    \end{scope}

    \node at (3.5, -1) {\huge$\Rightarrow$};

    \begin{scope}[xshift=5cm]
        \node[base] (i) at (0, 0) {$i$};
        \node[base] (j) [below = of i] {$j$};
        \node[base] (k) [right =  of j] {$k$};

        \draw[dedge] (i) to (j);
        \draw[dedge,violet] (j) to (k);

        \node[fitbox, fit=(i)] (0) {};
        \node[fitbox, fit=(j)] (0) {};
        \node[fitbox, fit=(k)] (0) {};
    \end{scope}
\end{tikzpicture}

         }
        \caption{Meek rule 1.}
    \end{subfigure}
    \hspace{\sep}
    \begin{subfigure}[b]{\size\textwidth}
        \centering
        \resizebox{\linewidth}{!}{\begin{tikzpicture}
    \begin{scope}
        \node[base] (i) at (0, 0) {$i$};
        \node[base] (j) [below = of i] {$j$};
        \node[base] (k) [right =  of j] {$k$};

        \draw[dedge] (i) to (j);
        \draw[dedge] (j) to (k);
        \draw[uedge] (i) to (k);

        \node[fitbox, fit=(i)] (0) {};
        \node[fitbox, fit=(j)] (0) {};
        \node[fitbox, fit=(k)] (0) {};
    \end{scope}

    \node at (3.5, -1) {\huge$\Rightarrow$};

    \begin{scope}[xshift=5cm]
        \node[base] (i) at (0, 0) {$i$};
        \node[base] (j) [below = of i] {$j$};
        \node[base] (k) [right =  of j] {$k$};

        \draw[dedge] (i) to (j);
        \draw[dedge] (j) to (k);
        \draw[dedge,violet] (i) to (k);

        \node[fitbox, fit=(i)] (0) {};
        \node[fitbox, fit=(j)] (0) {};
        \node[fitbox, fit=(k)] (0) {};
    \end{scope}
\end{tikzpicture}
         }
        \caption{Meek rule 2.}
    \end{subfigure}
    \hfill
    \begin{subfigure}[b]{\size\textwidth}
        \centering
        \resizebox{\linewidth}{!}{\def\padding{0.3cm}
\begin{tikzpicture}
    \begin{scope}
        \node[base] (i) at (0, 0) {$i$};
        \node[base] (j) [below = of i] {$j$};
        \node[base] (k) [right =  of j] {$k$};
        \node[base] (l) [above =  of k] {$l$};

        \draw[uedge] (i) to (j);
        \draw[uedge] (i) to (l);
        \draw[uedge] (i) to (k);
        \draw[dedge] (j) to (k);
        \draw[dedge] (l) to (k);

        \draw[fitbox]
        ($ (i.north west) + (-\padding, \padding) $) --
        ($ (j.south west) + (-\padding, -\padding) $) --
        ($ (j.south east) + (\padding, -\padding) $) --
        ($ (i.south east) + (\padding, -\padding) $) --
        ($ (l.south east) + (\padding, -\padding) $) --
        ($ (l.north east) + (\padding, \padding) $) --
        cycle;
        \node[fitbox, fit=(k)] (0) {};
    \end{scope}

    \node at (3.5, -1) {\huge$\Rightarrow$};

    \begin{scope}[xshift=5cm]
        \node[base] (i) at (0, 0) {$i$};
        \node[base] (j) [below = of i] {$j$};
        \node[base] (k) [right =  of j] {$k$};
        \node[base] (l) [above =  of k] {$l$};

        \draw[uedge] (i) to (j);
        \draw[uedge] (i) to (l);
        \draw[dedge,violet] (i) to (k);
        \draw[dedge] (j) to (k);
        \draw[dedge] (l) to (k);

        \draw[fitbox]
        ($ (i.north west) + (-\padding, \padding) $) --
        ($ (j.south west) + (-\padding, -\padding) $) --
        ($ (j.south east) + (\padding, -\padding) $) --
        ($ (i.south east) + (\padding, -\padding) $) --
        ($ (l.south east) + (\padding, -\padding) $) --
        ($ (l.north east) + (\padding, \padding) $) --
        cycle;
        \node[fitbox, fit=(k)] (0) {};

    \end{scope}
\end{tikzpicture}
         }
        \caption{Meek rule 3.}
    \end{subfigure}
    \hspace{\sep}
    \begin{subfigure}[b]{\size\textwidth}
        \centering
        \resizebox{\linewidth}{!}{\begin{tikzpicture}
    \begin{scope}
        \node[base] (i) at (0, 0) {$i$};
        \node[base] (j) [below = of i] {$j$};
        \node[base] (k) [right =  of j] {$k$};
        \node[base] (l) [above =  of k] {$l$};

        \draw[dedge] (i) to (j);
        \draw[uedge] (i) to (l);
        \draw[uedge] (i) to (k);
        \draw[dedge] (j) to (k);
        \draw[uedge] (l) to (k);

        \node[fitbox, fit=(i) (l)] (0) {};
        \node[fitbox, fit=(j)] (0) {};
        \node[fitbox, fit=(k)] (0) {};
    \end{scope}

    \node at (3.5, -1) {\huge$\Rightarrow$};

    \begin{scope}[xshift=5cm]
        \node[base] (i) at (0, 0) {$i$};
        \node[base] (j) [below = of i] {$j$};
        \node[base] (k) [right =  of j] {$k$};
        \node[base] (l) [above =  of k] {$l$};

        \draw[dedge] (i) to (j);
        \draw[uedge] (i) to (l);
        \draw[uedge] (i) to (k);
        \draw[dedge] (j) to (k);
        \draw[dedge,violet] (l) to (k);

        \node[fitbox, fit=(i) (l)] (0) {};
        \node[fitbox, fit=(j)] (0) {};
        \node[fitbox, fit=(k)] (0) {};
    \end{scope}
\end{tikzpicture}
         }
        \caption{Meek rule 4.}
    \end{subfigure}
    \hfill
    \caption{
        Meek rules \citep{meek1995}.
        The dashed lines partition nodes into sets based on
        ancestral relationships.
        Only Meek rule 1 establishes an ancestor for a node that
        previously had none, forcing a component to split.
    }
    \label{fig:meek-set-rule}
\end{figure}

 }

\begin{definition}[Set $\vstructset^m$]\label{def:vstructset}
    For all $w \in \bar{S}$, let $w \in \vstructset^m$ if and only if $ {u \CI
    v \mid S \cup W} $ and $ {u \not\CI v\mid S \cup W\cup\{w\}}$ for
    some $u \in V $, $ v \in \bar{S} $ and $ W \subset \bar{S}$
    with $ |W| = m $.
\end{definition}

These CI tests identify a downstream node $w$ by its properties as a
collider, or a descendant of a collider.
Specifically, conditioning on $w$ creates a dependency between
two previously independent nodes, $u$ and $v$.
Definition~\ref{def:vstructset} generalizes the concept of
\emph{minimal} conditional dependence
\citep{claassen2011characterization,magliacane2016ancestral},
by incorporating a prefix node set.

Next, we define tests to identify nodes whose incident edges are
oriented based on
Meek rules.

\begin{definition}[Set $\meekoneset^m$]\label{def:meekset}
    For all $w \in \bar{S}$, let $w \in \meekoneset^m$ if and only if
    $u\not\CI w\mid S$ and $u\CI w\mid S \cup W$ for some $u\in S$
    and $ W \subset \bar{S} $ with $ |W| = m$.
\end{definition}

This definition checks for a node $w$ that is initially dependent
on a node $u$ in the prefix node set $S$, but becomes independent once
we condition on $W \subseteq \bar{S}$.
Graphically, this means that all active paths between $ u $ and $
w $ pass through $ W $.
This implies $ w $ has an ancestor in $\bar{S}$, indicating that
it can be excluded when expanding the prefix node set $S$.

With these definitions, we can expand a prefix node set $ S \subset V
$ (which may be empty) into a new, larger set $ S' \subseteq V $.

\begin{restatable}{theorem}{prefix-set}\label{thm:prefix-set}
    Let $ S \subset V $ be a prefix node set.
    For any integer $\ell$, the set $S'=V\setminus (\vstructset^0
        \cup \vstructset^1 \cup \meekoneset^1 \cup \dots \cup
    \vstructset^\ell \cup \meekoneset^\ell)$ can be obtained with $
    \mathcal{O}(p^{\ell+3}) $ number of CI tests.
    In addition, $S'$ is a prefix node set satisfying $(S \cup
    \src(\bar{S})) \subseteq S' $, that is, it contains all the
    remaining source nodes in $\bar{S}$.
\end{restatable}

The proof of Theorem~\ref{thm:prefix-set} relies on the following lemmas.
These lemmas serve two purposes: first, they establish the
properties of the sets previously defined, and second, they
guarantee that $S'$ strictly increases the size of the prefix node set $S$.

\begin{restatable}{lemma}{vstructsetproperty}\label{lem:vstruct-set-property}
    Let $ S $ be a prefix node set.
    If $w \in \vstructset^m$, then $\Des[w] \subseteq \vstructset^m$.
    Furthermore, we have ${\vstructset^m \cap \src(\bar{S}) = \varnothing}$.
\end{restatable}

\begin{figure}[t]
    \centering
    \conditionalresize{\def\@dist{0cm}
\begin{tikzpicture}
    \node[base] at (-3, -2) (0) {0};
    \node[base] at (-3, 0) (1) {1};
    \node[base] at (0, -2) (2) {2};

    \node[f-node] at (3, -2) (3) {3};
    \node[d-node] at (3, 0) (4) {4};

    \draw[dedge] (0) -- (2);
    \draw[dedge] (1) -- (2);
    \draw[dedge] (2) -- (3);
    \draw[dedge] (3) -- (4);
    \draw[dedge] (1) -- (4);

    \node[above = \@dist of 2] (5) {$ \mathrel{\rotatebox{90}{$\in$}} $};
    \node[above = -0.15cm of 5] (6) {$ \src(\bar{S}) $};
    \node[right = \@dist of 3] (7) {$ \in \meekoneset^1 $};
    \node[right = \@dist of 4] (8) {$ \in \vstructset^1 \setminus \meekoneset^1 $};

    \node[] at (-3, -1) (S) {\textbf{S}};
    \node[fitbox, fit=(0) (1) (S)] (box) {};
\end{tikzpicture}

 }
    \caption{
        Example graph illustrating the classification of nodes into
        sets $ \vstructset^1 $ and $ \meekoneset^1 $ given the prefix node
        set $ S = \{0,1\} $.
    }
    \label{fig:f-descendants}
\end{figure}
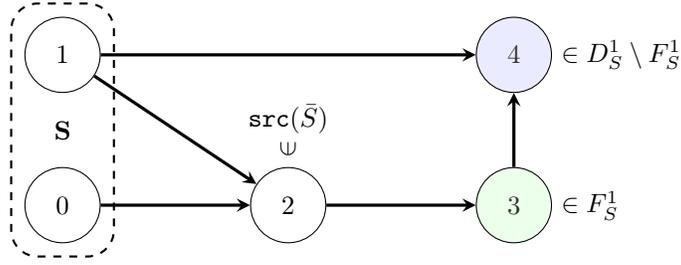

The preceding lemma shows that $\vstructset^m$ contains all
descendants of its members.
While descendants of nodes in $ \meekoneset^m $ are not necessarily
in $\meekoneset^m$, they must be in
$\vstructset^m\cup\meekoneset^m$, as shown in the following lemma;
see also Figure~\ref{fig:f-descendants}.

\begin{restatable}{lemma}{meeksetproperty}\label{lem:meek-set-property}
    Let $ S $ be a prefix node set.
    If $ w \in \meekoneset^m \setminus \left(\meekoneset^1 \cup \dots \cup
    \meekoneset^{m-1}\right)$, then $ \Des[w] \subseteq \meekoneset^{m} \cup
    \vstructset^m$.
    Furthermore, $ \meekoneset^m \cap \src(\bar{S}) = \varnothing $.
\end{restatable}

\begin{proof}[Proof of Theorem \ref{thm:prefix-set}]
    We first show that $ S' $ satisfies $ {\src(\bar{S})
    \subset S'} $.
    By
    Lemmas~\ref{lem:vstruct-set-property}~and~\ref{lem:meek-set-property}
    we have $(\vstructset^0 \cup
        \vstructset^1 \cup \meekoneset^1 \cup \dots \cup \vstructset^\ell
    \cup \meekoneset^\ell) \cap \src(S) = \varnothing $.
    Since $ \bar{S'} = \vstructset^0 \cup \vstructset^1 \cup
    \meekoneset^1 \cup \dots \cup \vstructset^\ell \cup \meekoneset^\ell $, it
    must hold that $ \src(\bar{S}) \subset S' $.

    Next we show that $ S' $ is a prefix node set.
    For this we only need to show that for all $ x \in S' $ and $ y \in
    \Anc(x) $, it holds that $ y \in S' $.
    For contradiction, assume $ x \in S' $ but $ y \notin S' $.
    We must have $ y \in \vstructset^i $ or $ y \in \meekoneset^j $
    for some $ i, j \leq m $.
    If $ y \in \vstructset^i $, then by
    Lemma~\ref{lem:vstruct-set-property} we have $ x \in \vstructset^i $, a
    contradiction.
    Otherwise, let $ j $ be the lowest number such that $ y \in
    \meekoneset^j $; then by Lemma \ref{lem:meek-set-property} we have $ x
    \in \meekoneset^j \cup \vstructset^j $, a contradiction.
    Therefore $ S' $ must be a prefix node set.

    We now bound the number of CI tests needed to obtain $S'$.
    Towards this, we derive the complexity of $ \vstructset^m $ and $\meekoneset^m$.
    Constructing $\vstructset^m$ requires checking, for all $w\in\bar{S}$,
    whether there exist $u \in V$, $v \in \bar{S} $, and $W \subset
    \bar{S}$ with $|W| = m$ such that $u\CI v \mid S \cup
    W$ and $u \not\CI v \mid S \cup W \cup \{w\}$.
    The required number of conditional independence tests is
    upper bounded by $2 \times \binom{|\bar{S}|}{1}
    \binom{|V| - 1}{1} \binom{|\bar{S}| - 1}{1}
    \binom{|\bar{S}| - 2}{m}$, which is in $
    \mathcal{O}(p^{m+3}) $.
    Therefore, computing $ \vstructset^m $ takes $
    \mathcal{O}(p^{m+3}) $ CI tests; similarly, $ \meekoneset^m $
    takes $ \mathcal{O}(p^{m+2}) $ CI tests.
    Since we compute up to $ m = \ell $, the total number of CI tests
    required is $ \mathcal{O}(p^{\ell+3}) $.
\end{proof}
 
\subsection{Learning the essential graph via
\texttt{\algoname}}\label{sec:algorithm}

\begin{algorithm}[htpb]
    \caption{{\algofullname} (\texttt{\algoshortname})}
    \label{alg:learning}
    \begin{algorithmic}[1]
        \Require CI queries from a distribution $\mathbb{P}$ that
        respects a DAG $ \mathcal{G} $
        \Ensure The essential graph of $ \mathcal{G} $

        \State $ \mathcal{S} \gets \text{empty ordered list}
        $;\label{algoline:init-start}
        \State $ S \gets \varnothing$;
        \State $ \mathcal{E} \gets \text{complete undirected graph on
        } V $;\label{algoline:init-end}

        \LComment{Prefix node set expansion}
        \While{$S\neq V$}\label{algoline:prefix-set-start}

        \State $ V' \gets V \setminus S$;
        \State $ \ell \gets 0 $;

        \While {there exists a clique in $ \mathcal{E}[V'] $ with
        size $ \ge \ell $}

        \LComment{Edge removal}
        \For {$ v \in V, w \in V' $ s.t. $ v - w $ in $ \mathcal{E}
        $}
        \State Remove $ v - w $ in $ \mathcal{E} $ if $ \exists W
        \subset V' $ with $ |W| = \ell $ s.t. $ {w \CI v \mid S \cup
        W} $;\label{algoline:cis}
        \EndFor

        \LComment{Ordering learned from v-structures}
        \State Compute $ \vstructset^\ell $;
        \State $ V' \gets V' \setminus \vstructset^\ell$;

        \If{$ \ell > 0 $}
        \LComment{Ordering learned from Meek rules}
        \State Compute $ \meekoneset^\ell $;
        \State $ V' \gets V' \setminus \meekoneset^\ell $;
        \EndIf

        \State $ \ell \gets \ell + 1 $;
        \EndWhile

        \State $ S \gets S \cup V' $;
        \State Add $ V' $ to the end of $ \mathcal{S} $;

        \EndWhile\label{algoline:prefix-set-end}

        \For{$v \in S_i, w \in S_j$ with $i<j$}\label{algoline:final-loop-start}
        \State If $v - w$ in $ \mathcal{E} $, replace with $ v \rightarrow w $;
        \EndFor\label{algoline:final-loop-end}

        \State \Return $\mathcal{E}$\label{algoline:end}
    \end{algorithmic}
\end{algorithm}

Our algorithm for learning essential graphs,
\algofullname \ (\texttt{\algoshortname}), is presented in
Algorithm~\ref{alg:learning}.
It extends the prefix node set method of \texttt{CCPG}
\citep{shiragur2024causal}.
Unlike \texttt{CCPG}, which only identifies a coarse,
partition-level graph, \texttt{\algoshortname} recovers
the complete essential graph.
In particular, we generalize their Definitions  4.1--4.3 by
introducing the parameter $m$; the original definitions are
recovered as special cases from our
Definitions~\ref{def:vstructset}~and~\ref{def:meekset} by
choosing $m=0$ or $m=1$. In the following, we prove
that this generalization is sufficient to identify all edges in
the essential graph, and we establish a stopping criterion for
Algorithm~\ref{alg:learning}, allowing us to bound the number of
CI tests needed.

\texttt{\algoname} initializes with a complete undirected graph $
\mathcal{E} $ on the node set $ V $.
It proceeds iteratively, constructing a sequence of expanding prefix
node sets $ {\varnothing \subset \dots \subset S \subset \dots
\subset V }$ by greedily adding elements according to
Theorem~\ref{thm:prefix-set}, while
simultaneously removing edges from the working graph.
The number of CI tests used in this construction is bounded by an
exponential function of $\ell$, which we show using our stopping
criterion---to be no greater than the size of the maximum undirected
clique in the essential graph.
A detailed, step-by-step example in Appendix~\ref{app:algo-example}
illustrates the execution of Algorithm~\ref{alg:learning} on a 5-node
graph, showing how the prefix node set is expanded, how edges are removed,
and how the final edge orientations are determined.

The sets $ \vstructset^m $ and $ \meekoneset^m $ are calculated using
the CI test results from the edge removal step (Line~\ref{algoline:cis}).
The computation of $ \vstructset^m $, for example, only requires some
targeted additional tests to identify descendants of v-structures
(see Algorithm~\ref{alg:vstructset}).
The specific implementation is detailed in Appendix~\ref{app:code}.

We now sketch the proof of Theorem~\ref{thm:algorithm}, which
establishes the correctness and guarantees for
Algorithm~\ref{alg:learning}.
The full formal proof can be found in Appendix~\ref{app:algorithm}.

\textbf{Correctness.} To show that the algorithm correctly outputs the essential graph $
\mathcal{E}(\mathcal{G}) $, we formally characterize the nodes
contained within the computed sets $ \vstructset^m $ and $ \meekoneset^m $ (Lemmas~\ref{lem:first-in-vstructset} and~\ref{lem:undirected-clique}).
This characterization demonstrates that these sets accurately capture
the necessary information for the algorithm to learn the partial
ordering defined by directed edges $\mathcal{E}(\mathcal{G})$.
This ordering is reflected in the partition $S_1, \dots, S_m$
identified by Algorithm~\ref{alg:learning}, where directed edges in $\mathcal{E}(\mathcal{G})$ exist only between nodes belonging to
different components~$S_i$~and~$S_j$.

\textbf{Number of CI tests.}
The key is to bound the maximum order $ \ell $ up to which the sets $
\vstructset^m $ and $ \meekoneset^m $ must be computed.
In Appendix~\ref{app:algorithm}, we prove the following result: if $ w \in \vstructset^m $ and $
w \notin \vstructset^0 \cup \dots \cup \vstructset^{m-1} $, then
there must exist a clique of size $ m $ in $ \Anc(w) \setminus S $.
This implies that the maximum required order is bounded by the
maximum clique in any of the components $ S_i $.
This leads to the overall complexity bound of $ {p^{\mathcal{O}(s)}} $
where $ s $ is the size of the maximum undirected clique in $
\mathcal{E}(\mathcal{G}) $.

 \section{Lower bound}\label{sec:lower-bound}

The following lemma from \citet{zhang2024membership} specifies the
exact CI tests on which two DAGs will disagree if they differ only by
a single edge, provided that edge is undirected in the essential
graph of the larger DAG.

\begin{lemma}[Lemma 7 in \citet{zhang2024membership}]\label{lem:testing}
    Let $ \mathcal{G} $ and $ \mathcal{H} $ be two DAGs such that
    $ \mathcal{H} $ differs from $ \mathcal{G}$ by missing
    one edge $ u \rightarrow v $, which is undirected in $
    \mathcal{E}(\mathcal{G}) $.
    Denote the maximum undirected clique in $ \mathcal{G} $
        containing $ u,v $ by $ S $.
        For disjoint sets $A,B,C\subset [p]$, the statement of
        whether $A$ and $B$ are d-separated by $C$ differs between
        $\mathcal{G}$ and $\mathcal{H}$ only if
    \begin{equation}
        (\Pa_{\mathcal{G}}(v) \cap \Ch_{\mathcal{G}}(u)) \cap S
        \subseteq C \cap S \subseteq (\Pa_{\mathcal{G}}(v)
        \setminus \{u\}) \cap S.
    \end{equation}
\end{lemma}

\begin{proof}[Proof of Theorem \ref{thm:lower}]
    First, note that by the binomial theorem we have $ 2^{s} =
    \sum_{i=0}^{s}\binom{s}{i}$. This number is exactly how many ways
    we can choose an unordered subset of any size from a fixed set of
    $ s $ elements.
    Therefore $ 2^s - \binom{s}{s-1} - \binom{s}{s} =2^s-s-1$ is exactly all
    the ways we can choose a subset of size less or equal to $ s-2 $
    from $ s $ nodes.

    Let $ S $ be the largest undirected clique in $ \mathcal{E}(\mathcal{G}) $.
    If less than $ 2^s -  s - 1 $ tests have been done,
    there must be a subset $ W \subseteq S $ of $ s-2 $ nodes or
    less, such that no test $ A \CI B \mid C $ has been performed
    where $ C \cap S = W $.
Since $ S $ is an undirected clique in $ \mathcal{E}(\mathcal{G})
    $, its nodes can be ordered arbitrarily to obtain a valid DAG in
    $ [\mathcal{G}] $.
    Let $ \mathcal{F} \in [\mathcal{G}] $ be a DAG such that in the topological
    ordering of $ S $ all nodes in $ W $ are placed consecutively
    after the first node.
    Let $ u $ denote this first node in the topological ordering, and let
    $ v $ denote the node after all nodes in $ W $.
    Let $ \mathcal{H} $ be the DAG obtained by removing the edge $ u
    \rightarrow v $ from $ \mathcal{F} $.
    For disjoint sets $ A, B, C $ testing if $ A $ and $ B $ are
    independent given $ C $ will disagree between $ \mathcal{F} $ and
    $ \mathcal{H} $ only if $ C \cap S = W $.
    This is because using Lemma \ref{lem:testing} we have
    \begin{equation}
        \begin{split}
            W &= (\Pa_{\mathcal{F}}(v) \cap \Ch_{\mathcal{F}}(u)) \cap
            S \subseteq C \cap S
            \ifbool{aistats}{\\ &}{}
            \subseteq (\Pa_{\mathcal{F}}(v) \setminus \{u\}) \cap S = W.
        \end{split}
    \end{equation}
    Since no CI tests have been performed that
    satisfy $C \cap S = W$, all the performed tests agree with $
    [\mathcal{G}] $ and $ [\mathcal{H}] $.
    Therefore, more tests must be performed in order to distinguish
    the MEC classes $[\mathcal{G}]$ and $ [\mathcal{H}]$.
\end{proof}

We remark that our lower bound in Theorem~\ref{thm:lower} is strictly
stronger than the one provided in \citet{zhang2024membership}.
Our result provides an \emph{any-case} guarantee: for any collection of
fewer than $2^{\Omega(s)}$ CI tests, we show an indistinguishable MEC
is guaranteed to exist.
In contrast, \citet{zhang2024membership} establish a \emph{worst-case}
result by constructing a specific MEC and proving that no algorithm
can distinguish it with fewer than $2^{\Omega(s)}$ tests.
 \section{Experiments}\label{sec:experiments}

In this section, we empirically evaluate our proposed algorithm, \algofullname
\ (\texttt{\algoshortname}),
and its variant, \texttt{\algoshortname+}, which incorporates $
\mathcal{O}(p^2) $ additional CI tests.  We observe that these
additional CI tests improve
performance in practice (details can be found in
Appendix~\ref{app:variant}).
We also benchmark performance against established causal discovery algorithms
on synthetic data from linear Gaussian structural equation
models~\citep{wright1921correlation} on random Erd\H{o}s-R\'enyi
graphs \citep{erdos1959random}, semi-synthetic gene expression data
obtained using the SERGIO simulator~\citep{dibaeinia2020sergio}, and
a real-world dataset.
Implementation details and data descriptions are provided in
Appendices~\ref{app:code} and \ref{app:experiments}, respectively.
Appendix~\ref{app:additional-experiments} presents additional
experiments, including results on scale-free graphs, networks with
thousands of nodes, and comparisons of the number of CI tests
performed by different algorithms.
All code for these experiments is available at
\url{https://github.com/uhlerlab/greedy-ancestral-search}.

\subsection{Linear Gaussian synthetic data}

\def\neighborhoodcaption{
    Comparison across Erd\H{o}s-R\'enyi graphs on 50 nodes and
    increasing expected neighborhood size.
    Results are averaged over 3 runs.
    (a) Execution time (seconds) presented on a logarithmic scale.
    (b) Accuracy comparison measured by Structural {H}amming
    Distance (SHD) between predicted and ground-truth graphs
    normalized by the number of possible edges.
    Shaded regions show the standard deviation.
}

\begin{figure}[h!]
    \centering
    \begin{subfigure}[b]{\ifbool{aistats}{0.42}{0.48}\textwidth}
        \centering
        \resizebox{\linewidth}{!}{\includegraphics{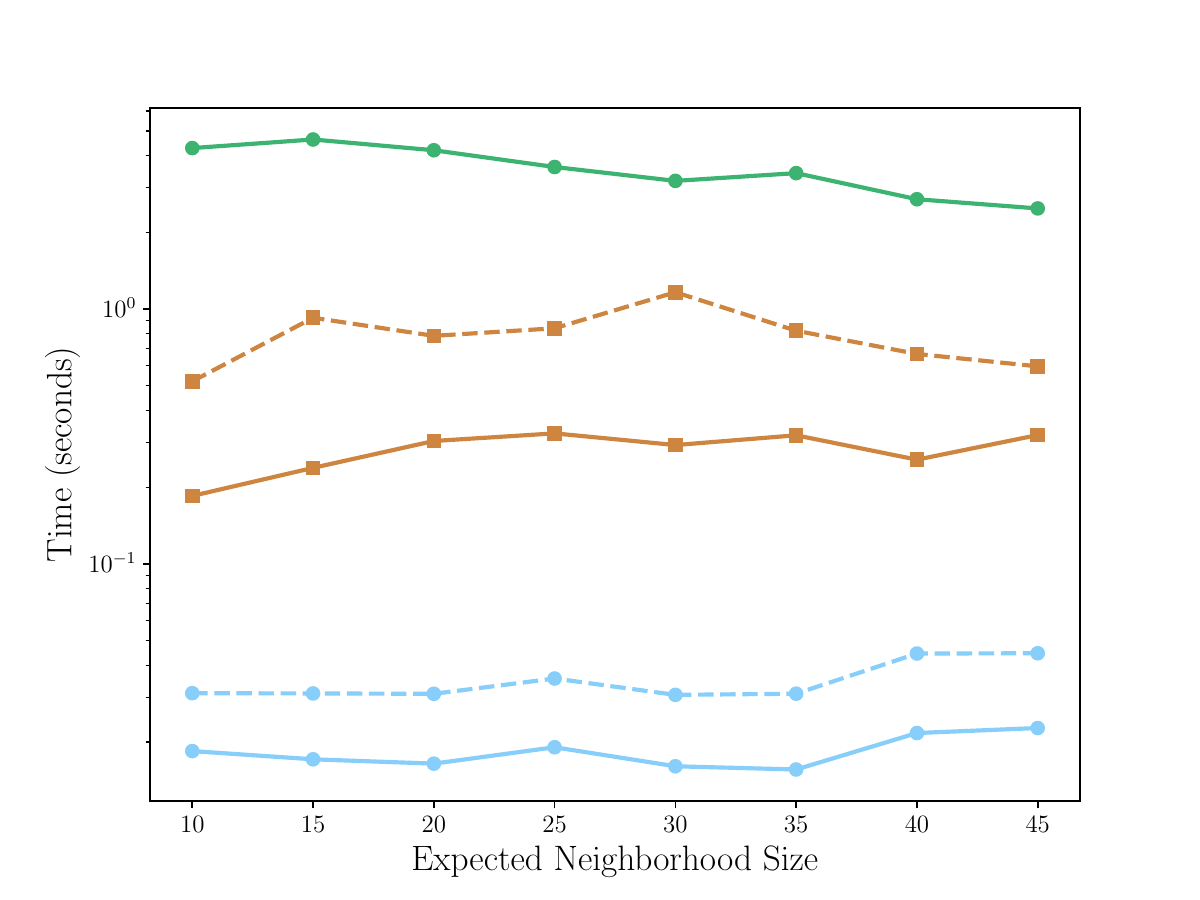}
        }
        \caption{}
    \end{subfigure}
    \hfill
    \begin{subfigure}[b]{\ifbool{aistats}{0.42}{0.48}\textwidth}
        \centering
        \resizebox{\linewidth}{!}{\includegraphics{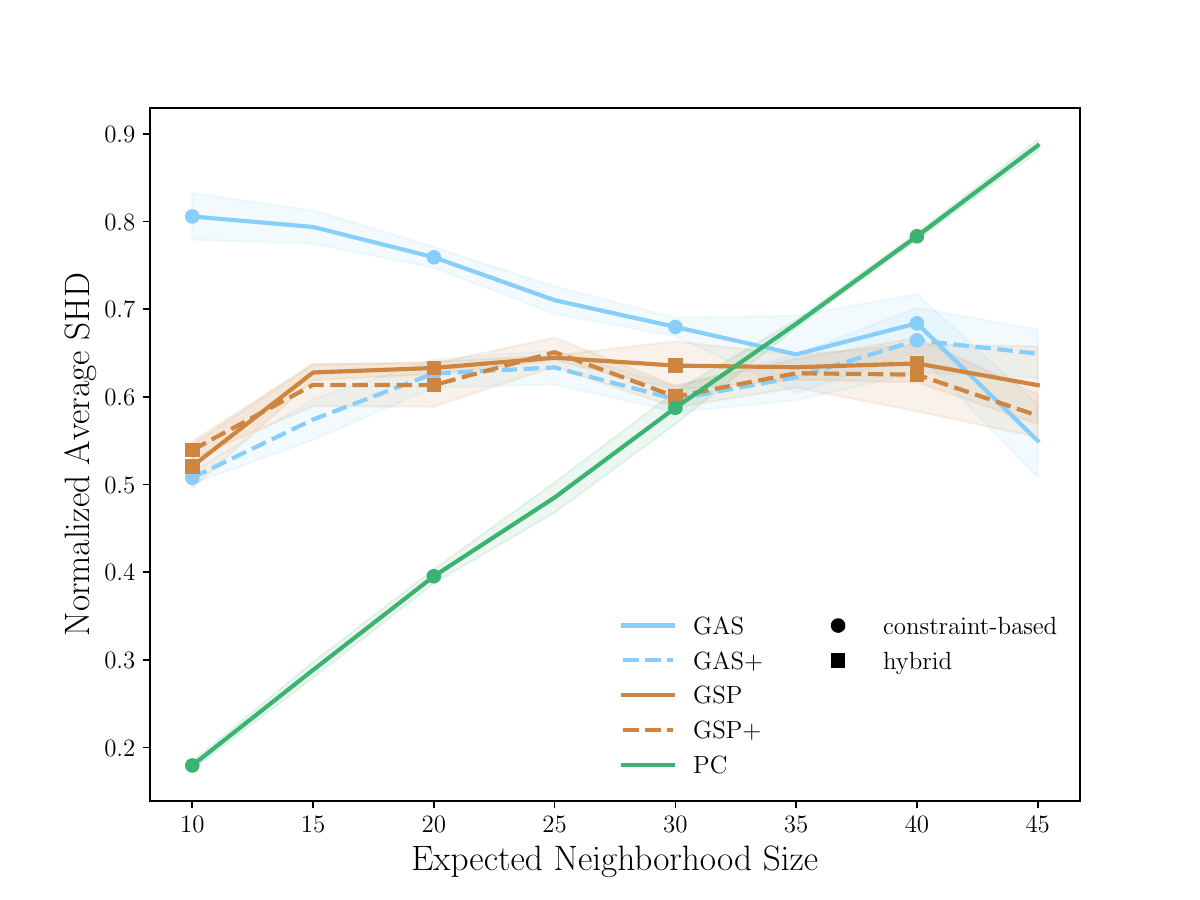}
        }
        \caption{}
        \label{fig:neighborhood-accuracy}
    \end{subfigure}
    \caption{\neighborhoodcaption}
    \label{fig:neighborhood}
\end{figure}

We compare both implementations of Algorithm \ref{alg:learning},
\texttt{\algoshortname} and \texttt{\algoshortname+}, against other
constraint-based and hybrid algorithms that rely on CI tests.
We include the popular constraint-based method \texttt{PC}
(\citealp{spirtes2001}), using the order-independent variant from
\citet{colombo2014order}.
We additionally include \texttt{GSP} with both depth $ 4 $ and depth
$ \infty $ (\citealp{solus2021consistency}).
A comparison with algorithms \texttt{FCI} and \texttt{GRaSP}$_2$
(\citealp{spirtes2001,lam2022greedy}), along with additional metrics
and a similar experiment scaled to 2500 nodes, can be found in
Appendix~\ref{app:additional-experiments}.

To ensure fair comparisons, evaluations with algorithms
originating from different softwares are presented separately.
Within each evaluation set, \texttt{\algoshortname} and
\texttt{\algoshortname+} used the same CI tester as the respective
baseline algorithms.

The experimental results are presented in Figure~\ref{fig:neighborhood}.
Regarding computational speed, a consistent finding is that both variants
of our algorithm, \texttt{\algoshortname} and \texttt{\algoshortname+}, are
significantly faster than the other methods.
Analyzing the prediction accuracy shown in
Figure~\ref{fig:neighborhood-accuracy}, \texttt{\algoshortname+} achieves
accuracy comparable to the \texttt{GSP} algorithms.
The \texttt{PC} algorithm demonstrates high accuracy in sparser
graphs, but its performance degrades by a large margin in denser graphs.
In contrast, \texttt{\algoshortname} performs fewer tests,
resulting in the highest accuracy in denser graphs but a lower
accuracy in sparse settings.

\subsection{SERGIO}

For our semi-synthetic evaluation, we simulated gene expression data
using the SERGIO model (\citealp{dibaeinia2020sergio}).
This enables precise comparison of inferred structures against
the ground truth DAG.
As in \citet{dibaeinia2020sergio}, we
utilized their DS1 graph, which samples $ 2700 $ cells each with the
expression of $ 100 $ genes.
Further details on this experimental setup are provided in
Appendix~\ref{app:experiments}.

\begin{figure}[htpb]
    \centering
    \begin{subfigure}[b]{0.22\textwidth}
        \centering
        \resizebox{\linewidth}{!}{\includegraphics{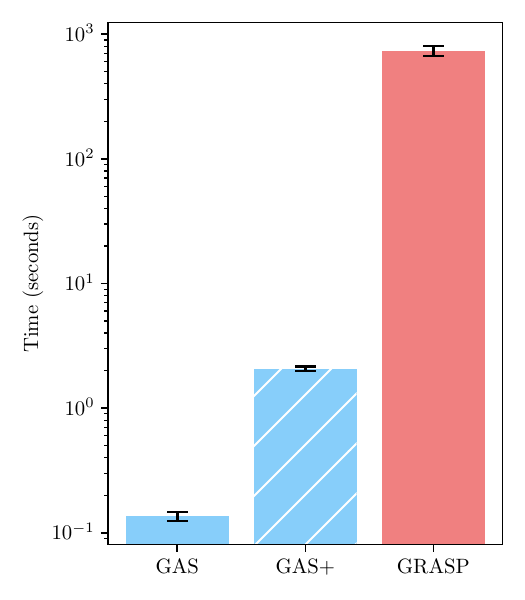}
        }
        \caption{}
    \end{subfigure}
    \begin{subfigure}[b]{0.22\textwidth}
        \centering
        \resizebox{\linewidth}{!}{\includegraphics{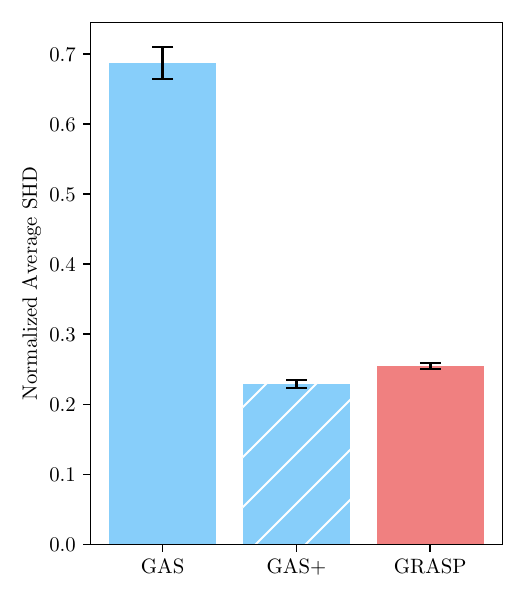}
        }
        \caption{}
    \end{subfigure}
    \caption{
        Comparison on semi-synthetic data generated using SERGIO.
        Results are averaged over 5 runs.
        (a) Execution time (seconds) presented on a logarithmic scale.
        (b) Accuracy comparison measured by Structural {H}amming
        Distance (SHD) between predicted and ground-truth graphs
        normalized by the number of possible edges.
    }
    \label{fig:sergio}
\end{figure}

Figure~\ref{fig:sergio} presents the comparative results on this
semi-synthetic data. The high-degree structure of this dataset makes
standard algorithms
like \texttt{PC} and \texttt{FCI} computationally infeasible due to
prohibitive runtimes (further discussed in
Appendix~\ref{app:experiments-sergio}). We therefore benchmark our
proposed methods, \texttt{\algoshortname}
and \texttt{\algoshortname+}, against \texttt{GRaSP}$_2$
(\citealp{lam2022greedy}).\footnote{
    \texttt{GSP}'s performance was found to be less competitive
    compared to the methods presented here.
}
The findings indicate that while \texttt{\algoshortname} is the fastest
algorithm overall, \texttt{\algoshortname+} achieves the highest accuracy
in recovering the underlying causal structure while being much faster than
\texttt{GRaSP}$_2$.

\subsection{Real-world example}

We consider the 6-variable Airfoil dataset
(\citealp{asuncion2007uci}); see \citet{lam2022greedy}.
Lacking a complete ground-truth DAG, we assess consistency
on the known causal relations for this system: (1) \emph{velocity},
\emph{chord} and \emph{attack} are experimental variables and thus
exogenous, while (2) \emph{pressure} is the measured outcome,
expected to be an effect of the other variables.

Figure~\ref{fig:airfoil} shows the partially directed graphs
learned by different algorithms on the Airfoil dataset. While all
algorithms correctly identify \emph{pressure} as a
downstream variable, \texttt{\algoname} is the most efficient,
requiring the fewest CI tests and running twice as
fast as the score-based \texttt{GRaSP}$_2$ algorithm.
However, we note that only \texttt{PC} correctly identifies
\emph{attack} as an exogenous variable.
For \texttt{GAS} and \texttt{GAS+}, this discrepancy stems from
their reliance on conditional dependencies to infer ancestry.
Specifically, the data indicates that \emph{frequency} and
\emph{chord} are marginally independent but become dependent when
conditioned on \emph{attack}.
If we assume the distribution respects an underlying DAG, this
pattern implies that \emph{attack} is a collider (or a descendant
of one) and thus non-exogenous.
While the CI test results may be correct, this implication can be
broken by latent confounders.

\begin{figure}[htpb]
    \centering
    \begin{subfigure}[b]{\ifbool{aistats}{0.17}{0.23}\textwidth}
        \centering
        \resizebox{\linewidth}{!}{\begin{tikzpicture}
    \def\radius{2cm}
    \def\angleoffset{90}

    \node (vel)   at (\angleoffset:\radius)      {Velocity};
    \node (att)   at (\angleoffset-300:\radius)  {Attack};
    \node (chord) at (\angleoffset-60:\radius)  {Chord};

    \node (freq)  at (\angleoffset-240:\radius)   {Frequency};
    \node (disp)  at (\angleoffset-120:\radius)  {Displacement};

    \node (press) at (\angleoffset-180:\radius)  {Pressure};

    \draw[uedge] (vel) -- (freq);
    \draw[dedge] (vel) -- (press);

    \draw[uedge] (chord) -- (disp);
    \draw[dedge] (chord) -- (att);
    \draw[dedge] (chord) -- (press);

    \draw[uedge] (freq) -- (disp);
    \draw[dedge] (freq) -- (press);
    \draw[dedge] (freq) -- (att);

    \draw[dedge] (disp) -- (press);
    \draw[dedge] (disp) -- (att);

    \draw[dedge] (att) -- (press);

\end{tikzpicture}
         }
        \caption{\texttt{\algoshortname}}
        \label{fig:airfoil-algo}
    \end{subfigure}
    \hfill
    \begin{subfigure}[b]{\ifbool{aistats}{0.17}{0.23}\textwidth}
        \centering
        \resizebox{\linewidth}{!}{\begin{tikzpicture}
    \def\radius{2cm}
    \def\angleoffset{90}

    \node (vel)   at (\angleoffset:\radius)      {Velocity};
    \node (att)   at (\angleoffset-300:\radius)  {Attack};
    \node (chord) at (\angleoffset-60:\radius)  {Chord};

    \node (freq)  at (\angleoffset-240:\radius)   {Frequency};
    \node (disp)  at (\angleoffset-120:\radius)  {Displacement};

    \node (press) at (\angleoffset-180:\radius)  {Pressure};

    \draw[uedge] (vel) -- (freq);
    \draw[dedge] (vel) -- (press);
    \draw[dedge] (vel) -- (att);

    \draw[uedge] (chord) -- (disp);
    \draw[dedge] (chord) -- (att);
    \draw[dedge] (chord) -- (press);

    \draw[uedge] (freq) -- (disp);
    \draw[dedge] (freq) -- (press);
    \draw[dedge] (freq) -- (att);

    \draw[dedge] (disp) -- (press);
    \draw[dedge] (disp) -- (att);

    \draw[dedge] (att) -- (press);

\end{tikzpicture}
         }
        \caption{\texttt{\algoshortname+}}
        \label{fig:airfoil-algo-plus}
    \end{subfigure}
    \hfill
    \begin{subfigure}[b]{\ifbool{aistats}{0.17}{0.23}\textwidth}
        \centering
        \resizebox{\linewidth}{!}{\begin{tikzpicture}
    \def\radius{2cm}
    \def\angleoffset{90}

    \node (vel)   at (\angleoffset:\radius)      {Velocity};
    \node (att)   at (\angleoffset-300:\radius)  {Attack};
    \node (chord) at (\angleoffset-60:\radius)  {Chord};

    \node (freq)  at (\angleoffset-240:\radius)   {Frequency};
    \node (disp)  at (\angleoffset-120:\radius)  {Displacement};

    \node (press) at (\angleoffset-180:\radius)  {Pressure};

    \draw[uedge] (chord) -- (att);
    \draw[uedge] (chord) -- (disp);
    \draw[uedge] (disp) -- (att);

    \draw[dedge] (disp) -- (press);
    \draw[dedge] (att) -- (freq);
    \draw[dedge] (vel) -- (freq);
    \draw[dedge] (chord) -- (press);
    \draw[dedge] (vel) -- (press);
    \draw[dedge] (freq) -- (press);
\end{tikzpicture}
         }
        \caption{\texttt{PC}}
        \label{fig:airfoil-pc}
    \end{subfigure}
    \hfill
    \begin{subfigure}[b]{\ifbool{aistats}{0.17}{0.23}\textwidth}
        \centering
        \resizebox{\linewidth}{!}{\begin{tikzpicture}
    \def\radius{2cm}
    \def\angleoffset{90}

    \node (vel)   at (\angleoffset:\radius)      {Velocity};
    \node (att)   at (\angleoffset-300:\radius)  {Attack};
    \node (chord) at (\angleoffset-60:\radius)  {Chord};

    \node (freq)  at (\angleoffset-240:\radius)   {Frequency};
    \node (disp)  at (\angleoffset-120:\radius)  {Displacement};

    \node (press) at (\angleoffset-180:\radius)  {Pressure};

    \draw[uedge] (vel) -- (freq);

    \draw[dedge] (freq) -- (att);
    \draw[dedge] (chord) -- (att);
    \draw[dedge] (vel) -- (att);

    \draw[dedge] (chord) -- (disp);
    \draw[dedge] (att) -- (disp);

    \draw[dedge] (vel) -- (press);
    \draw[dedge] (freq) -- (press);
    \draw[dedge] (disp) -- (press);
    \draw[dedge] (chord) -- (press);
    \draw[dedge] (att) -- (press);
\end{tikzpicture}
         }
        \caption{\texttt{GRaSP}$_2$}
        \label{fig:airfoil-grasp}
    \end{subfigure}
    \caption{
        Partially directed graphs learned in the Airfoil dataset.
        \texttt{\algoname}, \texttt{\algoname+}, and \texttt{PC}
        performed $46$, $50$, and $104$ distinct CI tests, respectively.
    }
    \label{fig:airfoil}
\end{figure}
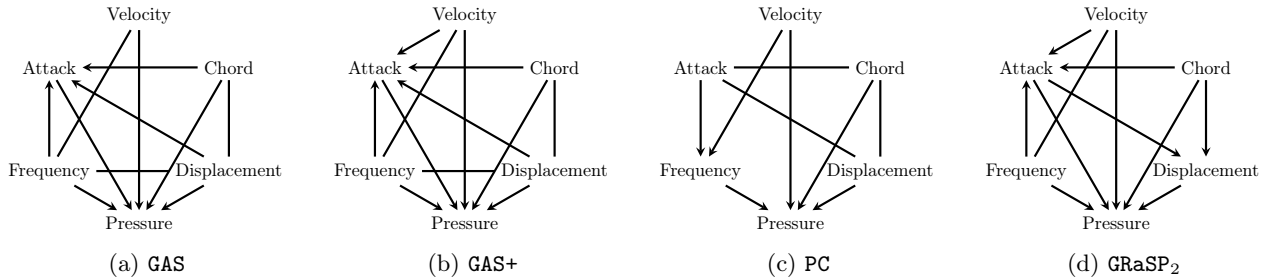

 \section{Discussion}\label{sec:discussion}

In this work, we established tight bounds on the number of
conditional independence tests required for causal structure discovery using constraint-based algorithms.
We propose an exponent-optimal algorithm achieving this bound up to a
logarithmic factor, requiring less
CI tests than existing algorithms.
Our empirical evaluations demonstrate that the proposed algorithm
is significantly faster than established baseline methods and
achieves highly competitive accuracy, particularly with denser graphs.

\textbf{Limitations and future work.}
This work is motivated by the goal of improving the efficiency of
constraint-based algorithms, both in terms of computational speed and
correctness.
We characterize efficiency by the number of CI tests performed, which
directly impacts computational speed.
However, as discussed in Section~\ref{sec:related-work}, the number
of CI tests is related with, but not equivalent to, the exact
correctness conditions required by an algorithm.
While we have shown that our algorithm can succeed under assumptions
weaker than the standard Markov and faithfulness conditions, its
precise correctness guarantees---and how they compare to those of
existing algorithms---remain unknown.
Moreover, understanding how finite-sample effects influence these
correctness conditions is an important direction for future work,
critical to understanding the potential and limitations of
constraint-based causal discovery.
Finally, our algorithm's reliance on both conditional independencies
and dependencies raises practical considerations.
Future work should explore whether the algorithm's performance on
finite, real-world data can be optimized by using different p-value
thresholds or distinct statistical tests for each constraint.
 
\subsubsection*{Acknowledgements}
We thank Kiran Shiragur for initial discussions that inspired this
work. We also thank the anonymous reviewers for their helpful feedback.
M.F.M.~was partially supported by the Eric and Wendy Schmidt Center at the Broad Institute, Fundaci\'o Privada Mir-Puig, and a MOBINT-MIF grant.
J.Z.~was partially supported by an Apple AI/ML PhD Fellowship.
C.U.~was partially supported by NCCIH/NIH (1DP2AT012345), ONR (N00014-24-1-2687), the United States Department of Energy (DE-SC0023187), and the Eric and Wendy Schmidt Center at the Broad Institute. 

\bibliography{refs}

\newpage
\appendix

\section{Meek rules}

\begin{proposition}[Meek rules, \cite{meek1995}]\label{prop:meek}
    We can infer all directed edges in $\mathcal{E}(\mathcal{G})$
    using the following four rules:
    \begin{enumerate}
        \item[R1.] If $i\rightarrow j \sim k$ and $i\not\sim k$, then
            $j\rightarrow k$.
        \item[R2.] If $i\rightarrow j \rightarrow k$ and $i\sim k$, then
            $i\rightarrow k$.
        \item[R3.] If $i\sim j, i\sim k, i\sim l, j\rightarrow k,
            l\rightarrow k$ and $j\not\sim l$, then $i\rightarrow k$.
        \item[R4.] If $i\sim j, i\sim k, i\sim l, j\rightarrow k,
            i\rightarrow j$ and $j\not\sim l$, then $l\rightarrow k$.
    \end{enumerate}
\end{proposition}

\ifbool{aistats}{
    \begin{figure}[h!]
    \def\size{0.35}
    \def\sep{4em}
    \centering
    \begin{subfigure}[b]{\size\textwidth}
        \centering
        \resizebox{\linewidth}{!}{\begin{tikzpicture}
    \begin{scope}
        \node[base] (i) at (0, 0) {$i$};
        \node[base] (j) [below = of i] {$j$};
        \node[base] (k) [right =  of j] {$k$};

        \draw[dedge] (i) to (j);
        \draw[uedge] (j) to (k);

        \node[fitbox, fit=(i)] (0) {};
        \node[fitbox, fit=(j) (k)] (0) {};
    \end{scope}

    \node at (3.5, -1) {\huge$\Rightarrow$};

    \begin{scope}[xshift=5cm]
        \node[base] (i) at (0, 0) {$i$};
        \node[base] (j) [below = of i] {$j$};
        \node[base] (k) [right =  of j] {$k$};

        \draw[dedge] (i) to (j);
        \draw[dedge,violet] (j) to (k);

        \node[fitbox, fit=(i)] (0) {};
        \node[fitbox, fit=(j)] (0) {};
        \node[fitbox, fit=(k)] (0) {};
    \end{scope}
\end{tikzpicture}

         }
        \caption{Meek rule 1.}
    \end{subfigure}
    \hspace{\sep}
    \begin{subfigure}[b]{\size\textwidth}
        \centering
        \resizebox{\linewidth}{!}{\begin{tikzpicture}
    \begin{scope}
        \node[base] (i) at (0, 0) {$i$};
        \node[base] (j) [below = of i] {$j$};
        \node[base] (k) [right =  of j] {$k$};

        \draw[dedge] (i) to (j);
        \draw[dedge] (j) to (k);
        \draw[uedge] (i) to (k);

        \node[fitbox, fit=(i)] (0) {};
        \node[fitbox, fit=(j)] (0) {};
        \node[fitbox, fit=(k)] (0) {};
    \end{scope}

    \node at (3.5, -1) {\huge$\Rightarrow$};

    \begin{scope}[xshift=5cm]
        \node[base] (i) at (0, 0) {$i$};
        \node[base] (j) [below = of i] {$j$};
        \node[base] (k) [right =  of j] {$k$};

        \draw[dedge] (i) to (j);
        \draw[dedge] (j) to (k);
        \draw[dedge,violet] (i) to (k);

        \node[fitbox, fit=(i)] (0) {};
        \node[fitbox, fit=(j)] (0) {};
        \node[fitbox, fit=(k)] (0) {};
    \end{scope}
\end{tikzpicture}
         }
        \caption{Meek rule 2.}
    \end{subfigure}
    \hfill
    \begin{subfigure}[b]{\size\textwidth}
        \centering
        \resizebox{\linewidth}{!}{\def\padding{0.3cm}
\begin{tikzpicture}
    \begin{scope}
        \node[base] (i) at (0, 0) {$i$};
        \node[base] (j) [below = of i] {$j$};
        \node[base] (k) [right =  of j] {$k$};
        \node[base] (l) [above =  of k] {$l$};

        \draw[uedge] (i) to (j);
        \draw[uedge] (i) to (l);
        \draw[uedge] (i) to (k);
        \draw[dedge] (j) to (k);
        \draw[dedge] (l) to (k);

        \draw[fitbox]
        ($ (i.north west) + (-\padding, \padding) $) --
        ($ (j.south west) + (-\padding, -\padding) $) --
        ($ (j.south east) + (\padding, -\padding) $) --
        ($ (i.south east) + (\padding, -\padding) $) --
        ($ (l.south east) + (\padding, -\padding) $) --
        ($ (l.north east) + (\padding, \padding) $) --
        cycle;
        \node[fitbox, fit=(k)] (0) {};
    \end{scope}

    \node at (3.5, -1) {\huge$\Rightarrow$};

    \begin{scope}[xshift=5cm]
        \node[base] (i) at (0, 0) {$i$};
        \node[base] (j) [below = of i] {$j$};
        \node[base] (k) [right =  of j] {$k$};
        \node[base] (l) [above =  of k] {$l$};

        \draw[uedge] (i) to (j);
        \draw[uedge] (i) to (l);
        \draw[dedge,violet] (i) to (k);
        \draw[dedge] (j) to (k);
        \draw[dedge] (l) to (k);

        \draw[fitbox]
        ($ (i.north west) + (-\padding, \padding) $) --
        ($ (j.south west) + (-\padding, -\padding) $) --
        ($ (j.south east) + (\padding, -\padding) $) --
        ($ (i.south east) + (\padding, -\padding) $) --
        ($ (l.south east) + (\padding, -\padding) $) --
        ($ (l.north east) + (\padding, \padding) $) --
        cycle;
        \node[fitbox, fit=(k)] (0) {};

    \end{scope}
\end{tikzpicture}
         }
        \caption{Meek rule 3.}
    \end{subfigure}
    \hspace{\sep}
    \begin{subfigure}[b]{\size\textwidth}
        \centering
        \resizebox{\linewidth}{!}{\begin{tikzpicture}
    \begin{scope}
        \node[base] (i) at (0, 0) {$i$};
        \node[base] (j) [below = of i] {$j$};
        \node[base] (k) [right =  of j] {$k$};
        \node[base] (l) [above =  of k] {$l$};

        \draw[dedge] (i) to (j);
        \draw[uedge] (i) to (l);
        \draw[uedge] (i) to (k);
        \draw[dedge] (j) to (k);
        \draw[uedge] (l) to (k);

        \node[fitbox, fit=(i) (l)] (0) {};
        \node[fitbox, fit=(j)] (0) {};
        \node[fitbox, fit=(k)] (0) {};
    \end{scope}

    \node at (3.5, -1) {\huge$\Rightarrow$};

    \begin{scope}[xshift=5cm]
        \node[base] (i) at (0, 0) {$i$};
        \node[base] (j) [below = of i] {$j$};
        \node[base] (k) [right =  of j] {$k$};
        \node[base] (l) [above =  of k] {$l$};

        \draw[dedge] (i) to (j);
        \draw[uedge] (i) to (l);
        \draw[uedge] (i) to (k);
        \draw[dedge] (j) to (k);
        \draw[dedge,violet] (l) to (k);

        \node[fitbox, fit=(i) (l)] (0) {};
        \node[fitbox, fit=(j)] (0) {};
        \node[fitbox, fit=(k)] (0) {};
    \end{scope}
\end{tikzpicture}
         }
        \caption{Meek rule 4.}
    \end{subfigure}
    \hfill
    \caption{
        Meek rules \citep{meek1995}.
        The dashed lines partition nodes into sets based on
        ancestral relationships.
        Only Meek rule 1 establishes an ancestor for a node that
        previously had none, forcing a component to split.
    }
    \label{fig:meek-set-rule}
\end{figure}

 }{}
 \section{Omitted proofs of upper bound}\label{app:upper-bound}

We will start by showing the following lemma.

\begin{restatable}{lemma}{lemPrefixMotivation}\label{lem:prefix-set}
    Let $ S \subset V $ be a prefix node set.
    Let $ A,B $ and $ C $ be disjoint subsets such that $ A \subset V
    $, and $ B, C \subset \bar{S} $.
    Let $ \mathcal{H} = \mathcal{G}[\bar{S} \cup A] $ be the induced
    subgraph of $ \mathcal{G} $ on the node set $ \bar{S} \cup A $.
    Then $ A \not\CI_{\mathcal{G}} B \mid (S \setminus A) \cup C $ if
    and only if $ A \not\CI_{\mathcal{H}} B \mid C $.
\end{restatable}
\begin{proof}
    We will first prove the \emph{if} direction.
    Assume $ A \not\CI_{\mathcal{H}} B \mid C $.
    This means there exists an active path $ P $ in $ \mathcal{H} $
    between a node in $ A $ and a node in $ B $, given $ C $.
    By definition of $ \mathcal{H} $ as $ \mathcal{G}[\bar{S} \cup A]
    $, all nodes on the path $ P $ necessarily belong to $ \bar{S} \cup A
    $, which means no node on $ P $ is part of $ S \setminus A $.
    For $ P $ to be active in $ \mathcal{H} $ given $ C $, every
    non-collider on $ P $ must not be in $ C $, and every collider on
    $ P $ must either be in $ C $ or have a descendant in $ C $.
    Now, considering this same path $ P $ within the graph $
    \mathcal{G} $ and conditioned on $ (S \setminus A) \cup C $: any
    non-collider on $ P $, not being in $ C $ (due to being
    active in $ \mathcal{H} $) and not being in $ S \setminus A $
    (due to being fully in $ \mathcal{H} $) is therefore not in $ (S
    \setminus A) \cup C $.
    Similarly, any collider on $ P $ that is in $
    \Anc_{\mathcal{H}}[C] $ is consequently in $
    \Anc_{\mathcal{G}}[(S \setminus A) \cup C]$.
    Since path $ P $ satisfies the criteria for an active path
    in $ \mathcal{G} $ when conditioned on $ (S \setminus A) \cup C
    $, it follows that $ A \not\CI_{\mathcal{G}} B \mid (S \setminus
    A) \cup C $.

    For the \emph{only if} direction, we assume ${A
    \not\CI_{\mathcal{G}} B \mid (S \setminus A) \cup C}$.
    This assumption implies there exists a path $ P $ between a node
    in $ A $ and a node in $ B $ that is active in $ \mathcal{G} $
    when conditioned on $ (S \setminus A) \cup C $; we can select $ P
    $ such that only its endpoints are in $ A \cup B $.
    For $ P $ to be active, any non-collider on this path must not be
    in $ (S \setminus A) \cup C $, which directly means non-colliders
    on $ P $ must reside in $ A \cup \bar{S} \setminus C $.
    Colliders on $ P $, on the other hand, must be in $
    \Anc_{\mathcal{G}}[(S \setminus A) \cup C] $.
    Consider a collider $ c $ on path $ P $.
    If $ c $ were in $ \Anc_{\mathcal{G}}[S \setminus A] $, then
    because $ S $ is a prefix set (which implies $
    \Anc_{\mathcal{G}}[S \setminus A] \subseteq S $), $ c $ itself
    must be in $ S $.
    As $ c $ is an intermediate node on $ P $ (not in $ A \cup B $),
    $ c $ being in $ S $ means $ c \in S \setminus A $.
    The argument then follows that if such a collider $ c \in S
    \setminus A $ is on $ P $, the prefix nature of $ S $
    necessitates that at least one of $ c $'s neighbors on path $ P $
    must also be in $ S \setminus A $ while being an intermediate
    node on $P$ as $B\subset\bar{S}$.
    This neighbor, however, must function as a non-collider on $ P $
    while also being in $ S \setminus A $, thereby blocking $ P $ in
    $\mathcal{G}$, a contradiction.
    Therefore, any collider on path $ P $ must be in $
    \Anc_{\mathcal{G}}[C] \setminus S \subseteq \Anc_{\mathcal{H}}[C]$.
    Given that non-colliders are in $ A \cup \bar{S} \setminus C $, and
    colliders are in $ \Anc_{\mathcal{H}}[C] $, it follows that $ P $
    is active in $ \mathcal{H} $ given $ C $.
    This implies $ A \not\CI_{\mathcal{H}} B \mid C $.
\end{proof}

We will make use of the following lemma.

\begin{lemma}\label{lem:min-sep-set}
    Let $ S \subset V $ be a prefix set.
    Let $ a,b \in V $ and $ W \subset \bar{S} $ be such that $ a
    \CI b \mid S \cup W $ and $ a \not\CI b \mid S \cup W' $ for all
    $ W' \subset W $.
    Then, $ W \subseteq \Anc(a) \cup \Anc(b) $.
\end{lemma}
\begin{proof}[]
    We aim to show that every element $ w \in W $ must be a member of
    $ \Anc(a) \cup \Anc(b) $.
    The argument can be understood by iteratively considering how
    elements of $ W $ contribute to blocking paths.
    If $ W = \varnothing $, the inclusion in $ \Anc(a) \cup \Anc(b) $
    is trivial.

    \begin{enumerate}
        \item
            \emph{Initial element.}
            Since $ W' = \varnothing $ is a proper subset of $ W $,
            we have $ a \not\CI b \mid S $.
            Therefore, there must exist an active path $ P $ between
            $ a $ and $ b $ given $ S $.
            Since $ S $ is prefix, by Lemma~\ref{lem:prefix-set},
            there are no intermediate nodes in $ S $ and therefore,
            the path does not contain colliders.
            This implies that all nodes on $ P $ are necessarily within $
            \Anc(a) \cup \Anc(b) $.
            Since $ a \CI b \mid S \cup W $, this path $ P $ must be
            blocked by $ S \cup W $.
            Thus, at least one element from $ W $, termed $ x_1 $,
            must lie on $ P $ as a non-collider.
            Being a node on $ P $, $ x_1 $ is therefore in $ \Anc(a)
            \cup \Anc(b) $.

        \item
            \emph{Subsequent elements.}
            Let $ X_k = \{x_1, \dots, x_k\} $ be a subset of $ W $
            for which it has been established that $ X_k \subseteq
            \Anc(a) \cup \Anc(b) $.
            If $ X_k \neq W $, we must have $ a \not\CI b \mid S \cup X_k $.
            So there exists an active path, $ P_k $, between $ a $
            and $ b $ given $ S \cup X_k $.
            Since $ S \cup X_k $ does not block $ P_k $, a node
            from $ W \setminus X_k $, termed $ x_{k+1} $, must
            block $ P_k $.
            Since the $ P_k $ is active, any collider on $ P_k $ must
            also be in $ \Anc[X_k] $.
            Since $ X_k \subseteq \Anc(a) \cup \Anc(b) $ we have $
            \Anc[X_k] \subseteq \Anc(a) \cup \Anc(b) $.
            This implies that all nodes on the entire path $ P_k $
            are in $ \Anc(a) \cup \Anc(b) $.
            Since $ x_{k+1} $ is on $ P_k $, $ x_{k+1}
            $ must also be in $ \Anc(a) \cup \Anc(b) $.
    \end{enumerate}

    Therefore every element in $ W $ must be in $ \Anc(a) \cup \Anc(b) $.
\end{proof}

\subsection{Proof of Lemma~{\ref{lem:vstruct-set-property}}}

To prove Lemma \ref{lem:vstruct-set-property}, we will make use of the
following definition and lemma.

\begin{definition}[Set $A_S$]
    Let $S \subset V $ be a set of nodes not necessarily prefix.
    For all $w \in \bar{S}$, let $w \in A_S$ if and only if $u \CI v \mid
    S$ and $ u \not\CI v \mid S \cup \{w\} $ for some $u \in V$ and
    $v \in \bar{S}$.
\end{definition}

\begin{lemma}\label{lem:as-property}
    If $w \in A_S$, then $\Des[w] \subset A_S$.
    Furthermore, $A_S\cap \src(\bar{S})=\varnothing$.
\end{lemma}
\begin{proof}[]
    We first show that if $w\in A_S$, then it must hold that
    $\Des[w]\subset A_S$: since $w\in A_S$, there exists a node
    $u\in V$ and $v\in \bar{S}$ such that $u \CI v\mid S$ and $u
    \not\CI v\mid S\cup\{w\}$.
    We now show that for any $x\in \Des[w]$, we have $u \not\CI v\mid
    S\cup\{x\}$.

    Since $u \not\CI v \mid S \cup \{w\}$, there is a path $P$ from $u$
    to $v$ that is active given $S\cup\{w\}$.
    Therefore, all non-colliders on $P$ are not in $S\cup\{w\}$ and
    all colliders on
    $P$ are in $\Anc[S\cup\{w\}]$.
    Since $x\in \Des[w]$, all colliders on $P$ are in $\Anc[S\cup\{x\}]$.
    If all non-colliders are not in $S\cup\{x\}$, then $P$ is an
    active path from $u$ to $v$ given $S\cup\{x\}$, and thus $u
    \not\CI v\mid S\cup\{x\}$.
    Otherwise there is a non-collider on $P$ that is $x$.

    Since $u\CI v\mid S$, the path $P$ is inactive given $S$. From
    above we know that all non-colliders on $P$ are not in $S$.
    Therefore there exists a collider on $P$ that is not in
    $\Anc[S]$.
    Suppose the leftmost and rightmost such colliders are $k,k'$ (it
    is possible that $k=k'$), then $k,k'$ must be in
    $\Anc[S\cup\{w\}]\setminus\Anc[S]\subseteq\Anc[w]\subset\Anc[x]$.
    Therefore, the path $ P $ takes the form: $ {u - \dots \rightarrow
    k \leftarrow \dots \rightarrow k' \leftarrow \dots - v} $.
    If $ x $ is between $ u $ and $ k $ (or between $ k' $ and $ v
    $), consider the path $ Q $ in the graph by cutting out the parts
    between $ x $ and $ k' $ (or between $ k $ and $ x $) on $ P $ and
    replacing them with directed edges from $ k' $ to $ x $ (or from
    $ k $ to $ x $).
    Compared to $P$, the additional non-colliders on $Q$ are all on
    the directed path from $k'$ to $x$ (or from $ k $ to $ x $).
    They are not in $S$ since $k,k'\notin \Anc[S]$, and thus $Q$ has
    no non-colliders in $S$.

    Compared to $P$, there is no collider on $P$ that is not in
    $\Anc[S]$ and is still on $Q$ by the fact that $k,k'$ are
    leftmost and rightmost colliders on $P$ that are not in
    $\Anc[S]$. Therefore, $x$ must be a collider on $Q$, or else $Q$
    is active given $S$ and $u\not\CI v\mid S$. Therefore all
    non-colliders on $Q$ are not in $S\cup\{x\}$. Every collider on
    $Q$ is either $x$ or a collider of $P$, which is in
    $\Anc[S\cup\{x\}]$.
    Thus $Q$ is active given $S\cup\{x\}$.
    Therefore $u\not\CI v\mid S\cup \{x\}$.

    If $ x $ is between $ k $ and $ k' $ we can replace the part between
    $ k $ and $ x $ by the direct path from $ k $ to $ x $ and the
    part between $ x $ and $ k' $ by the direct path from $ k' $ to $ x $.
    There is no additional collider on the new path and all
    non-colliders are not in $ S $ since $ k,k' \notin \Anc[S] $.
    Therefore the path is active given $S\cup\{x\}$ and $u\not\CI
    v\mid S\cup \{x\}$.

    Next we show that $A_S\cap \src(\bar{S})=\varnothing$: for
    contradiction assume that there exists a node $a \in
    \src(\bar{S})$ such that $a \not \in \bar{S} \backslash A_S$,
    that is $a \in \src(\bar{S})$ and for some node $u\in V, v \in
    \bar{S}$, ${u \CI v \mid S }$ and $ { u \not\CI v \mid S \cup \{a\} } $.

    Since $u \CI v \mid S $ and $ u \not\CI v \mid S \cup \{a\}$, there
    exists a path $P$ between $u$ and $v$ which is inactive when
    conditioned on $S$ but is active upon conditioning on $S \cup \{
    a\}$.
    Moreover, this path contains a node $b$ that is a
    collider on $P$ and satisfies: $a \in \Des[b]$ and $\Des[b] \cap
    S=\varnothing$. Since $\Des[b] \cap S=\varnothing$, we have that
    $b \in \bar{S}$.
    Furthermore, since $b \in \bar{S}, a \in \src(\bar{S})$ and $a
    \in \Des[b]$, this implies that $b=a$.
    Therefore, the path $P$ takes the form: $P=v - \dots \rightarrow
    a \leftarrow \dots - u$.
    All the colliders on the path $P$ either belong to or have
    descendant in the set $S \cup \{ a\}$.

    Now consider the path $v - \dots \to a$ and note that it is active
    given $S$.
    Let $k$ be the number of nodes between $v$ and $a$ on this
    path $v-v_1 - \dots -  v_k\to a$.
    It is immediate that $v_{k} \in S$ since $a \in \src(\bar{S})$.
    However, since $v_k \in S$, and since we condition on the set
    $S$, this should be a collider for the path $Q$ to be active,
    which is not possible.
    Thus, we get a contradiction, which completes the proof.
\end{proof}

\vstructsetproperty*
\begin{proof}[]
    Since $ w \in \vstructset^m $, there exists two nodes $ u \in V $, $ v
    \in \bar{S} $ and a set $ W \subset \bar{S} $ with $ |W| = m $
    such that $ {u \CI v \mid S \cup W} $ and $ {u\not\CI v\mid S
    \cup W\cup\{w\}}$.
    Let $ S' = S \cup W $.
    We have that $ w \in A_{S'} $ and by Lemma \ref{lem:as-property},
    $\Des[w] \subset A_{S'}$, $ A_{S'} \cap \src(\bar{S'}) = \varnothing $.
    Since $ A_{S'} \subset \vstructset^m $ we have that $ \Des[w]
    \subset \vstructset^m $.
    Additionally, $ \src(\bar{S}) \subset W \cup \src(\bar{S'})$ and
    $ w \notin W $, therefore $ w \notin \src(\bar{S}) $.
\end{proof}

\subsection{Proof of Lemma~{\ref{lem:meek-set-property}}}

To prove Lemma \ref{lem:meek-set-property} we will make use of the
following result.

\begin{lemma}\label{lem:meek-set-anc}
    Let $ S $ be a prefix set.
    Let $ w \in \bar{S} $ be such that $ w \notin \meekoneset^1 \cup \dots \cup
    \meekoneset^{m-1} $ and $ w \in \meekoneset^m $.
    Let $ u \in S $ and $ W \subset \bar{S} $ be
    such that $ u \not\CI w \mid S$ and $ u \CI w \mid S \cup W$.
    Then $ W \subset \Anc(w) $.
\end{lemma}
\begin{proof}[]
    Since $ u \not\CI w \mid S \cup W' $ for all $ W' \subset W $,
    by Lemma \ref{lem:min-sep-set} we have $ W \subset \Anc(u) \cup \Anc(w) $.
    However, given that $ u \in S $, $ S $ is a prefix set and $ W
    \subset \bar{S} $, it necessarily follows that $  W \subset \Anc(w) $.
\end{proof}

\meeksetproperty*
\begin{proof}[]
    We first show that if $w \in \meekoneset^m \setminus
    \meekoneset^1 \cup \dots \cup
    \meekoneset^{m-1}$, then for any $y\in\Des(w)$, we have $y \in
    \meekoneset^m \cup D_S^m$.
    Since $w\in \meekoneset^m$, we have $ u \not\CI w \mid S $ and $
    u \CI w \mid S \cup W $  for some $ u \in S $.

    Take the active path between $ u $ and $ w $ given $ S $ and
    extend it by the directed path from $ w $ to $ y $.
    Note that none of nodes on the directed path from $w$ to $y$
    are in $S$, since $S$ is prefix and $w \not\in S$.
    Therefore, this extended path is also active given $S$, which
    means $u\not\CI y\mid S$.

    Thus, if $y\not\in \meekoneset^m$, then it must hold that $u\not\CI y\mid
    S\cup W$.
    This means there is an active path, denoted by $P$, between $u$
    and $y$ given $S\cup W$.
    Consider extending this path by the directed path from $w$ to
    $y$, denoted as $Q$.
    Compared to $P$, the additional non-colliders on $Q$ are not in
    $S\cup W$: for $S$, this is because $S$ is prefix, $w\not\in
    S$, and all additional non-colliders are descendants of $w$; for
    $W$, this is because by Lemma~\ref{lem:meek-set-anc} we have $W
    \cap \Des(w) = \varnothing$.
    Thus, $Q$ is active given $S\cup W$, unless $y$ is a collider on $Q$.
    Since $u\CI w\mid S\cup W$, the path $Q$ must be inactive
    given $S\cup W$, which means $y$ is a collider on $Q$.
    This means $Q$ is active given $S\cup W \cup \{y\}$.
    Therefore, $u\not\CI w\mid S\cup W \cup \{y\}$.
    Together with $u\CI w\mid S\cup W$, we have $y\in D_S^m$.

    Next we show that $\meekoneset^m \cap \src(\bar{S})=\varnothing$.
    Since $w\in \meekoneset^m$, we have $ u \not\CI w \mid S $
    and $ u \CI w \mid S \cup W $  for some $ u \in S $ and $ W
    \subset \bar{S} $, $ |W| = m $.
    Let $ W' \subset \bar{S} $ be the smallest set such that $ u' \CI
    w \mid S \cup W' $ for some $ u' \in S $.
    We must have $ 1 \leq |W'| \leq |W| $.
    Additionally, Lemma~\ref{lem:meek-set-anc} shows that $ W'
    \subset \Anc(w) $.
    Since $ W' $ is a non-empty subset of $ \bar{S} $ containing
    ancestors of $ w $ we have $ w \notin \src(\bar{S}) $.
\end{proof}
 
\subsection{Proof of Theorem~{\ref{thm:algorithm}}}\label{app:algorithm}

We develop the proof of Theorem~\ref{thm:algorithm} through a series
of intermediate results.
Our first key lemma establishes that any node serving as the collider
in a v-structure is necessarily contained within a set $\vstructset^m$.

\begin{lemma}\label{lem:v-structure-test}
    Let $ S \subset V $ be a prefix node set.
    If there are three nodes $ u \in V $, $ k, v \in \bar{S} $
    that form a v-structure $ u \rightarrow k \leftarrow v $, then
    there exists a set $ W \subset \bar{S} $ such that $ u \CI v
    \mid S \cup W $ and $ u \not\CI v \mid S \cup W \cup \{k\}$.
\end{lemma}
\begin{proof}[]
    The path $ u \rightarrow k \leftarrow v $ is active given $ k $
    and therefore $ u \not\CI v \mid S \cup W \cup \{k\}$ for any $
    W \subset \bar{S} $.
    We will now see that for $ W = \left( \Pa(u) \cup \Pa(v) \right)
    \setminus S $ we have $ u \CI v \mid S \cup W $.
    If $ u \not\CI v \mid S \cup W $ then there exists an active
    path between $ u $ and $ v $ given $ S \cup W $.
    Let $ u - u' - \dots - v' - v $ be this path.
    If we have $ u \leftarrow u' $ or $ v' \rightarrow v $ then the
    path would be inactive.
    We therefore have $ u \rightarrow u' - \dots - v' \leftarrow v $
    and the path must contain at least one collider.
    Let $ c $ be a collider in this path, thus $ c \in \Anc[S \cup W] $.
    If $ c \in \Anc[S] $, since $ S $ is prefix, $ c \in S $.
    In this case, the nodes on the path adjacent to $ c $ are also in
    $ S $ and the path must be inactive.
    By acyclicity, if there is a collider $ c \in \Anc[W] $ in the
    path then there must be at least two colliders.
    Let $ c, c' $ be the leftmost and rightmost colliders on the path.
    We must have $ c \in \Anc[\Pa(v)] \setminus S $ and $ c' \in
    \Anc[\Pa(u)] \setminus S $ and the graph contains a cycle $
    c' \rightarrow \dots \rightarrow u \rightarrow \dots \rightarrow
    c \rightarrow \dots \rightarrow v \rightarrow c' $, a
    contradiction to $ \mathcal{G} $ being a DAG.
\end{proof}

We now provide a result characterizing specific nodes in $ \vstructset^m $.

\begin{lemma}\label{lem:first-in-vstructset}
    Let $ S $ be a prefix set.
    If $ w \in \vstructset^m $ and $ \Anc(w) \cap \vstructset^{m} =
    \varnothing $, then $ w $ must serve as the collider in a
    v-structure $ a \rightarrow w \leftarrow b $.
\end{lemma}
\begin{proof}[]
    Since $ w \in \vstructset^m $, there exists $ u \in V $, $ v \in
    \bar{S} $ and $ W \subset \bar{S} $ with $ |W| = m $ such that
    $ {u \CI v \mid S \cup W }$ and $ { u \not\CI v \mid S \cup W
    \cup \{w\} } $.
    Additionally, the active path between $ u $ and $ v $ given $ S
    \cup W \cup \{w\} $ must contain a v-structure $ a \rightarrow c
    \leftarrow b $ with $ c \in \Anc[w] \cap \vstructset^m $.
    But $ \Anc(w) \cap \vstructset^{m} = \varnothing $ and therefore $ w = c $.
    Thus, $ a \rightarrow w \leftarrow b $ is a v-structure.
\end{proof}

We then establish a relationship between the set $\meekoneset^m$ and
undirected cliques, which follows from the observation that
$\meekoneset^m$ includes downstream nodes of edges oriented by Meek rules.

\begin{lemma}\label{lem:undirected-clique}
    Let $ S $ be a prefix set.
    Let $ m > 0 $ and $ V' \subset \bar{S} $ such that,
    \begin{enumerate}
        \item There does not exist a v-structure $ a \rightarrow c
            \leftarrow b $ with $ c,b \in V' $.
        \item There is an edge between nodes in $ V' $ oriented by a Meek rule.
        \item $ S \cup V' $ is prefix.
        \item $ V' \cap (\meekoneset^1 \cup \dots \cup
            \meekoneset^{m-1}) = \varnothing $.
    \end{enumerate}
    Then there exists an undirected clique of size larger or equal to
    $ m $ in the subgraph induced by $ V' $ in $ \mathcal{E}(\mathcal{G}) $.
\end{lemma}
\begin{proof}[]
    Let $ V' \subset \bar{S} $ satisfy conditions (1)--(4).
    By condition (2), there exists an edge $ u \rightarrow v $ in $
    \mathcal{E}(\mathcal{G}) $ with $ u, v \in V' $ that is oriented
    by a Meek rule.

    The application of a Meek rule to orient an edge $ u \rightarrow v $
    requires at least one existing oriented edge between nodes in
    $ \Anc(v) $, as observed in Proposition~\ref{prop:meek}.
    Thus, the orientation of any edge using a Meek rule depends on
    prior orientations within $ \Anc(v) $.
    We can trace this chain of dependencies backward through ancestors.
    Since $ \mathcal{G} $ is acyclic and finite, this dependency
    chain must terminate.

    This termination guarantees the existence of an edge $ x
    \rightarrow y $ between two nodes $ x,y \in \bar{S} $ such that there
    are no prior oriented edges in $ \mathcal{E}(\mathcal{G}) $
    between nodes in $ \Anc(y) \cap \bar{S} $.
    Condition (3) implies $ x, y \in V' $ and condition (1) requires
    this edge to be oriented by a Meek rule.

    Thus, $ x \rightarrow y $ is oriented by directed edges between
    nodes in $ S $ and nodes in $ V' $ or directed edges between nodes in $ S $.
    We now determine which Meek rule must have oriented this edge.
    We proceed by eliminating rules 2, 3 and 4.
    \begin{itemize}
        \item[R2:]
            Requires $ z $ such that $ x \rightarrow z $ and $ z
            \rightarrow y $.
            Since $ x,z \in \Anc(y) $ and $ x \rightarrow z $, we
            must have $ z \in S $, but then, we must also have $ x \in S $,
            a contradiction.
        \item[R3:]
            Requires $ x,y $ to form a v-structure with another
            $ z \in V $, a contradiction.
        \item[R4:]
            Requires $ z, z' $ such that $ z \rightarrow z' $, $
            z' \rightarrow y $, $ z \sim x $, $ z' \sim x $ and $ z
            \not\sim y $.
            Since $ z, z' \in \Anc(y) $ we must have $ z \in S $.
            This together with $ z \sim x $ and $ z \not\sim y $
            means we have $ z \rightarrow x $, and therefore, $ x
            \rightarrow y $ is also oriented by Meek rule 1.
    \end{itemize}

    Therefore, $ x \rightarrow y $ is oriented by Meek R1.
    This implies there exists $ z \in S $ such that $ z \rightarrow x
    $ and $ z \not\sim y $ in $ \mathcal{E}(\mathcal{G}) $.

    By condition (4) we have $ y \notin \meekoneset^1 \cup \dots \cup
    \meekoneset^{m-1} $.
    Therefore,  $ z \not\CI y \mid S \cup W $ for any $ W \subset
    \bar{S} $ such that $ |W| < m $.
    Let $ K = \Pa(y) \setminus S $.
    Since $ z \in S $ we have $ \Pa(z) \subseteq S $, and therefore,
    $ z \CI y \mid S \cup K $.
    Thus, $ |K| \geq m $.

    By condition (1) all nodes in $ K $ must be adjacent.
    Finally, since we selected $ x \rightarrow y $ such that there is
    no directed edge in $ \mathcal{E}(\mathcal{G}) $ between nodes in
    $ \Anc(y) \cap \bar{S} $, then $ {K \subset \Anc(y) \cap \bar{S}
    \subset V'} $ must form an undirected clique in the essential
    graph of at least size $ m $.
\end{proof}

\begin{corollary}\label{cor:undirected-clique}
    Let $ S $ be a prefix set.
    Let $ m > 0 $ and $ V' \subset \bar{S} $ such that,
    \begin{enumerate}
        \item There does not exist a v-structure $ a \rightarrow c
            \leftarrow b $ with $ c,b \in V' $.
        \item There exists a clique $ W \subset V' $ of size larger
            or equal to $ m $.
        \item $ S \cup V' $ is prefix.
        \item $ V' \cap (\meekoneset^1 \cup \dots \cup
            \meekoneset^{m-1}) = \varnothing $.
    \end{enumerate}
    Then there exists an undirected clique of size larger or equal to
    $ m $ in the subgraph induced by $ V' $ in $ \mathcal{E}(\mathcal{G}) $.
\end{corollary}
\begin{proof}[]
    If $ W $ is undirected in the essential graph, we are done.

    Otherwise, if $ W $ is not undirected in the essential graph, we
    can apply a Meek rule to one of the edges in it.
    By Lemma \ref{lem:undirected-clique} there exists an undirected
    clique of size larger or equal to $ m $ in the subgraph induced
    by $ V' $ in $ \mathcal{E}(\mathcal{G}) $.
\end{proof}

We now extend the characterization of minimal separators in chordal
undirected graphs to DAGs.
The proof of this extension utilizes the moral graph construction and
the concept of separation in undirected graphs.
For a DAG $ \mathcal{G} $, its \emph{moral graph}, denoted $ \mathcal{G}^m $,
is the undirected graph created by adding edges between parents
sharing common children in $ \mathcal{G} $ and then removing all arrowheads.
In an undirected graph, a node set $ C $ \emph{separates} two
disjoint node sets $ A $ and $ B $ if all paths between $ A $ and $
B $ intersect $ C $.
The following result is a key component of our argument.

\begin{proposition}[Proposition 3.25 in
    \citet{lauritzen1996}]\label{prop:lauritzen}
    Let $ A, B$ and $ C $ be disjoints subsets of $ V $.
    Then $ C $ d-separates $ A $ from $ B $ in $ \mathcal{G} $ if and
    only if $ C $ separates $ A $ from $ B $ in $ \mathcal{G}[\Anc[A
    \cup B \cup C]]^m $.
\end{proposition}

\begin{lemma}{}\label{lem:min-sep-set-no-vstruct}
    Let $ S \subset V $ be a prefix node set.
    If for some $ u \in V $, $ v \in \bar{S} $ and $ W \subset
    \bar{S} $ we have $ u \CI v \mid S \cup W $, $ u \not\CI v \mid S
    \cup W' $ for all $ W' \subset \bar{S} $ such that $ |W'| <
    |W| $  and $ (\Anc[u] \cup \Anc[v]) \setminus S $ contains no
    v-structures then $ W $ is a clique in $ \mathcal{G} $.
\end{lemma}
\begin{proof}[]
    Let $ \mathcal{H} = \mathcal{G}[\bar{S} \cup \{u\}] $.
    By Lemma~\ref{lem:prefix-set} we have that $ u \CI_{\mathcal{G}}
    v \mid S \cup W' $ for some $ W' \subset \bar{S} $ if and only if
    $ u \CI_{\mathcal{H}} v \mid W' $.
    By Lemma~\ref{lem:min-sep-set} we have that $ W \subseteq \Anc(u)
    \cup \Anc(v) $ and therefore $ \Anc[\{u, v\} \cup W] =
    \Anc[\{u, v\}] $.
    Let $ \mathcal{F} = \mathcal{H}[\Anc[\{u,v\}]] $.
    Applying Proposition~\ref{prop:lauritzen} we have that for any $
    W' \subset \Anc(\{u,v\}) $, $ W' $ d-separates $ u $ and $ v $ in
    $ \mathcal{H} $ if and only if $ W' $ separates $ u $ and $ v $
    in $ \mathcal{F}^m $.
    As $ (\Anc[u] \cup \Anc[v]) \setminus S $ does not contain any
    v-structures and $ \mathcal{F}^m $ does not have any added adjacencies,
    any cycle of of four or more nodes will have a chord.
    Therefore, $ \mathcal{F}^m $ is a chordal graph.
    Since $ W $ is a minimal separating set in a chordal graph, by
    Theorem~1 in \citet{dirac1961} it is a clique in $ \mathcal{F}^m $.
    Since $ \mathcal{F}^m $ has the same adjacencies as $ \mathcal{F}
    $, $ W $ forms a clique in $ \mathcal{F} $ and also in $ \mathcal{G}
    $, because $ \mathcal{F} $ is a subgraph of $ \mathcal{G} $.
\end{proof}

The following result establishes that if a node belongs to $
\vstructset^m $, then, a clique must exist among its ancestors.

\begin{lemma}\label{lem:clique}
    Let $ S $ be a prefix set.
    If $ w \notin \vstructset^0 \cup \dots \cup \vstructset^{m-1} $
    and $ w \in \vstructset^m $ then, there exists a set of nodes $ W
    \subset \Anc(w) \cap \bar{S} \setminus \vstructset^{m} $ of size
    larger or equal to $ m $ that forms a clique in $ \mathcal{G} $.
\end{lemma}
\begin{proof}[]
    Let $ c $ be a node in $ \Anc[w] \cap \bar{S} $ that acts as a
    collider in a v-structure $ a \rightarrow c \leftarrow b $,
    with $ a \in V $, $ b \in \bar{S} $, chosen such that no node in $ \Anc(c) $
    also serves as such a collider.
    This node must exist due to $ \mathcal{G} $ being acyclic and
    having a finite amount of nodes.

    From Lemma~\ref{lem:v-structure-test}, we  know there exists a
    set $ W \subset \bar{S} $ that separates $ a $ and $ b $ given $
    S $, that is, $ a \CI b \mid S \cup W $.
    By the selection of $ c $, the set $ \Anc(c) \cap \bar{S} $
    contains no v-structures.
    As $ a, b \in \Anc(c) $ it follows that $ (\Anc[a] \cup \Anc[b])
    \setminus S $ contains no v-structures.
    Then, by Lemma~\ref{lem:min-sep-set-no-vstruct} the smallest set $ W $
    that satisfies $ a \CI b \mid S \cup W $ must form a clique.

    Now, because $ a \rightarrow c \leftarrow b $ forms a v-structure,
    conditioning on the collider $ c $ induces a dependence.
    Specifically $ a \not\CI b \mid S \cup W \cup \{c\} $ holds,
    where $ W $ is the smallest separating set identified previously.
    Furthermore, we are given $ w \notin \vstructset^0 \cup \dots \cup
    \vstructset^{m-1} $, and by Lemma~\ref{lem:vstruct-set-property},
    it follows that $ {c \notin\vstructset^0 \cup \dots \cup \vstructset^{m-1}}$.
    This condition on $ c $, implies that $ |W| \geq m $.
\end{proof}

Finally, we establish the guarantees of Algorithm~\ref{alg:learning},
as stated in Theorem~\ref{thm:algorithm}, through the sequence of
lemmas presented below.

\begin{lemma}\label{lem:algo-clique}
    In Algorithm~\ref{alg:learning}, if at any iteration the working
    graph $ \mathcal{E} $ contains a clique of size larger or equal
    to $ \ell $ in $ V' $ then $ \mathcal{E}(\mathcal{G}) $
    restricted to the nodes in $ V' $ also contains an undirected
    clique of size larger or equal to $ \ell $.
\end{lemma}
\begin{proof}
    We will prove that $ \mathcal{E}(\mathcal{G}) $ restricted to $
    V' $ contains an undirected clique of size larger or equal to $ \ell $.
    By Theorem~\ref{thm:prefix-set}, $ S \cup V' $ is a prefix set.

    If $ V' $ contains a v-structure, let $ a \rightarrow c
    \leftarrow b $ a v-structure such that $ c,b \in V' $, and such that
    there is no v-structure $ a' \rightarrow c' \leftarrow b' $ with
    $ c', b' \in \Anc(c) \cap \bar{S} $.
    Note that $ |V'| $ is bounded and therefore this v-structure must exist.
    By Lemma~\ref{lem:v-structure-test} we have $ c \in \vstructset^k
    $ for some $ k \geq \ell $.
    By Lemma~\ref{lem:clique} there exists a clique $ W $ of size larger or
    equal to $ m $ in $ \Anc(c) \cap \bar{S} $.
    We also have that $ S \cup (\Anc(c) \cap \bar{S}) $ is prefix and
    since $ c \notin \vstructset^0 \cup \dots \cup
    \vstructset^{\ell-1} \cup \meekoneset^1 \cup \dots \cup
    \meekoneset^{\ell-1} $, by Lemma~\ref{lem:meek-set-property} we
    have $ (\Anc(c) \cap \bar{S}) \cap (\meekoneset^1 \cup \dots \cup
    \meekoneset^{\ell-1}) = \varnothing $.
    By Corollary~\ref{cor:undirected-clique} there exists an undirected
    clique in $ \Anc(c) \cap \bar{S} $.

    If $ V' $ does not contain a v-structure we will prove by contradiction.
    Assume that there does not exist a clique of size larger or equal to $
    \ell $ in $ \mathcal{G}[V'] $.
    Then, we must have $ u \sim v $ in $ \mathcal{E} $ but not in $
    \mathcal{G} $ for some $ u,v \in V' $.
    This means that there exists a subset $ W \subset V' $, $ |W|
    \geq \ell $ such that $ u \CI v \mid S \cup W $ and is minimal.
    We have that $ (\Anc[u] \cup \Anc[v]) \setminus S $ does not
    contain v-structures and by
    Lemma~\ref{lem:min-sep-set-no-vstruct}, it must be a clique in $
    \mathcal{G} $.
    We have $ |W| \geq \ell $ and by
    Corollary~\ref{cor:undirected-clique} we will have an undirected
    clique in the essential graph restricted to $ V' $.
\end{proof}

\begin{lemma}\label{lem:vstruct-sets}
    Let $ S_1, \dots, S_m $ be the partition on $ V $ created
    by Algorithm~\ref{alg:learning} and let $ u \rightarrow c
    \leftarrow v $ be a v-structure in $ \mathcal{G} $.
    Then, $ u \in S_i $, $ v \in S_j $ and $ c \in S_k $ with $ k > i,j $.
\end{lemma}
\begin{proof}[]
    Without loss of generality, assume $ j \geq i $.
    Since $ S_1 \cup \dots \cup S_j $ is prefix by
    Theorem~\ref{thm:prefix-set}, we must have $ k \geq j $.
    If not $ k > j $ then $ k = j $.
    Let $ S = S_1 \cup \dots \cup S_{j-1} $.
    By Lemma~\ref{lem:v-structure-test} we have $ u \CI v \mid S \cup W $ and $
    u \not\CI v \mid S \cup W \cup \{c\} $ for some $ W \subset \bar{S} $.
    Let $ W' \subset \bar{S} $ be the set with the least amount of
    nodes that satisfies this and let $ k = |W'| $.
    We have that $ c \in \vstructset^k $ and $ c \notin \vstructset^1
    \cup \dots \cup
    \vstructset^{k-1} $.
    By Lemma~\ref{lem:clique} there exists a clique of size larger or
    equal to $ k $ in $ \Anc(c) \cap \bar{S} \subseteq S_j $,
    meaning, $ \vstructset^k $ must have been computed, a
    contradiction to $ c \in S_j $.
\end{proof}

\begin{lemma}\label{lem:algorithm-skeleton}
    The graph created by Algorithm~\ref{alg:learning} has the same
    skeleton as $ \mathcal{G} $.
\end{lemma}
\begin{proof}[]
    We will prove by contradiction.
    Let $ \mathcal{E} $ denote the graph created by
    Algorithm~\ref{alg:learning} and assume that $ \mathcal{E} $ and
    $ \mathcal{G} $ do not have the same skeleton.

    Let $ u,v \in V $.
    If $ u \sim v $ in $ \mathcal{G} $ we have $ u \not\CI v \mid W $
    for all $ W \subset V $.
    Therefore, if $ \text{sk}(\mathcal{G}) \neq
    \text{sk}(\mathcal{E}) $, there must exist $ u \in S_i $ and $ v
    \in S_j $ such that $ u \sim v $ in $ \mathcal{E} $ but not in $
    \mathcal{G} $.

    Denote $ S = S_1 \cup \dots \cup S_{j-1}. $
    There must exist a set $ W \subset S_j $ such that $ u \CI v
    \mid S \cup W $ and is minimal.
    By Lemma~\ref{lem:vstruct-sets}, $ (\Anc[u] \cup \Anc[v]) \setminus
    S $ does not contain any v-structures and by
    Lemma~\ref{lem:min-sep-set-no-vstruct} we have that $ W $ forms
    a clique.
    This contradicts Algorithm~\ref{alg:learning} not removing this edge.
\end{proof}

\begin{lemma}\label{lem:directed-edges}
    Let $ S_1, \dots, S_m $ be the partition on $ V $ created
    by Algorithm~\ref{alg:learning}.
    For all $ u, v \in V $ if $ u \rightarrow v $ in $
    \mathcal{E}(\mathcal{G}) $ then $ u \in S_i $ and $ v \in S_j $
    such that $ i < j $.
\end{lemma}
\begin{proof}[]
    We prove this by contradiction.
    Assume there exists a directed edge $ u \rightarrow v $ in $
    \mathcal{E}(\mathcal{G}) $ such that both $ u $ and $ v $ belong
    to the same component $ S_k $.

    By Lemma \ref{lem:vstruct-sets}, for any v-structure $ a \rightarrow c
    \leftarrow b $ in $ \mathcal{G} $, we will have $ a \in S_i, b
    \in S_j, c \in S_k $ with $ k > i,j $.
    This implies that if we have $ u \rightarrow v $ in $
    \mathcal{E}(\mathcal{G}) $ and $ u,v $ belong to the same
    component, the edge must have been oriented by a Meek rule.

    Let $ S = S_1 \cup \dots \cup S_{k-1} $ and let $ s' $ denote the
    largest clique in $ \mathcal{E}(\mathcal{G})[S_k] $.
    By Theorem~\ref{thm:prefix-set}, $ S $ is a prefix set.
    By Lemma \ref{lem:vstruct-sets}, there does not exist a
    v-structure $ a \rightarrow c \leftarrow b $ with $ c,b \in S_k $.
    By definition, we have $ S_k \cap (\vstructset^0 \cup \dots \cup
        \vstructset^{s'} \cup \meekoneset^1 \cup \dots \cup
    \meekoneset^{s'}) = \varnothing $ and by Lemma
    \ref{lem:undirected-clique}, there must exist a
    clique in $ \mathcal{E}(\mathcal{G})[S_k] $ of size strictly
    larger than $ s' $, a contradiction.
\end{proof}

\thmAlgorithm*
\begin{proof}[]
    We will show that Algorithm \ref{alg:learning} outputs the
    essential graph of $ \mathcal{G} $ and that it performs, at most, $
    p^{\mathcal{O}(s)} $ CI tests.

    Let $ \mathcal{E} $ denote the output of Algorithm \ref{alg:learning}.
    By Lemma \ref{lem:algorithm-skeleton}, $ v $ and $ w $ are
    adjacent in $ \mathcal{E} $ if and only if $ v $ and $ w $ are
    adjacent in $ \mathcal{G} $.
    We will now prove that if $ v \rightarrow w $ in $ \mathcal{E} $
    then also $ v \rightarrow w $ in $ \mathcal{G} $.
    If $ v \rightarrow w $ in $ \mathcal{E} $ we must have $ v \in
    S_i $ and $ w \in S_j $ for some $ i < j $.
    By Theorem \ref{thm:prefix-set}, $ S = S_1 \cup \dots \cup S_i $
    is prefix, therefore, we have $ v \rightarrow w $ in $ \mathcal{G} $.
    Finally, using Lemma \ref{lem:directed-edges}, we have $ v
    \rightarrow w $ in $ \mathcal{E} $ if and only if we have $ v
    \rightarrow w $ in $ \mathcal{E}(\mathcal{G}) $.
    Therefore, $ \mathcal{E} = \mathcal{E}(\mathcal{G}) $.

    By Lemma \ref{lem:algo-clique} the algorithm will calculate sets
    $ \vstructset^m $ and $ \meekoneset^m $ of up to $ m \leq s $ where $ s
    $ is the largest undirected clique in the essential graph of $
    \mathcal{G} $.
    By Theorem \ref{thm:prefix-set} an iteration requires at most $
    \mathcal{O}(p^{s+3}) $ CI tests.
    Thus, the algorithm uses $ \mathcal{O}(p^{s+4}) $ CI tests,
    proving Theorem \ref{thm:algorithm}.
\end{proof}
 \section{Step-by-step example of
Algorithm~\ref{alg:learning}}\label{app:algo-example}

\def\Xsep{1.5cm}
\def\Ysep{1cm}

We now provide a step-by-step example of Algorithm~\ref{alg:learning}.

\begin{example}
    Suppose we have a probability distribution $ \mathbb{P} $ that is
    Markov and faithful to the graph below.
    In this case, the graph and the essential graph are equal.
    \begin{figure}[h!]
        \centering
        \begin{tikzpicture}[node distance = \Ysep and \Xsep]
            \node[base] (0) {0};
            \node[base, above = of 0] (1) {1};
            \node[base, right = of 0] (2) {2};
            \node[base, right = of 2] (3) {3};
            \node[base, above = of 3] (4) {4};

            \path[dedge]
            (0) edge (2)
            (1) edge (2)
            (1) edge (4)
            (2) edge (3)
            (3) edge (4);
        \end{tikzpicture}
        \label{fig:example-algo-step}
    \end{figure}

    \vspace{0.5em}

    \setcounter{step}{-1}
    \begin{step}[Initialization, Lines
        \ref{algoline:init-start}--\ref{algoline:init-end}]
        We initialize an empty prefix node set $ S= \varnothing $,
        an empty list $ \mathcal{S} $, and a complete undirected
        graph $ \mathcal{E} $.
        \begin{figure}[h!]
            \centering
            \begin{tikzpicture}[node distance = \Ysep and \Xsep]
                \node[base] (0) {0};
                \node[base, above = of 0] (1) {1};
                \node[base, right = of 0] (2) {2};
                \node[base, right = of 2] (3) {3};
                \node[base, above = of 3] (4) {4};

                \path[uedge]
                (0) edge (1)
                (0) edge (2)
                (0) edge [bend right] (3)
                (0) edge (4)
                (1) edge (2)
                (1) edge (3)
                (1) edge (4)
                (2) edge (3)
                (2) edge (4)
                (3) edge (4);
            \end{tikzpicture}
        \end{figure}
    \end{step}

    \begin{step}[Prefix node set expansion, Lines
        \ref{algoline:prefix-set-start}--\ref{algoline:prefix-set-end}]
        We now learn the partial ordering of the essential graph
        through $ \vstructset^m $ and $ \meekoneset^m $, while
        establishing adjacencies.

        The expansion step  is repeated as long as the prefix node
        set does not contain all nodes ($S \neq V$).
        For each expansion, we work on a set of working nodes, denoted $ V' $.
        This set excludes the prefix set and any identified downstream nodes.
        In each expansion, we indicate the prefix node set $S$. The
        set of working nodes, $V'$, is specified only when it differs
        from $\bar{S}$ (the set of nodes not in $S$).

        For each $ \ell $, we remove edges between nodes if
        there exists a set $ W \subseteq V' $ with $ |W| = \ell $
        that renders them independent.
        Nodes identified as being in $\vstructset^\ell$ are colored
        \colorbox{\dcolor}{blue}, while nodes in $\meekoneset^\ell$
        are colored \colorbox{\fcolor}{green}.

        \begin{itemize}
            \item
                First prefix node set expansion ($ S = \varnothing,
                \mathcal{S} = \left[\right] $).

                Nodes $0$ and $1$ are found to be conditionally
                independent given the empty set, so the edge between
                them is removed.
                Furthermore, we find that conditioning on node $2$
                makes $0$ and $1$ dependent (unblocking the
                v-structure $0 \to 2 \leftarrow 1$).
                Similarly, conditioning on $3$ or $4$ (descendants of
                the collider) also makes $0$ and $1$ dependent.
                Therefore, we identify $\vstructset^0 = \{2,3,4\}$.
                The set of working nodes is then updated for the next
                iteration: $ V' = V' \setminus \vstructset^0 =
                \{0,1,2,3,4\} \setminus \{2,3,4\} = \{0,1\} $.
                For $ \ell = 1 $ no additional adjacencies are removed.
                The expansion stops here since no clique of size $ 2
                $ exists in the current working set $V' = \{0,1\}$.

                \begin{figure}[H]
                    \centering
                    \begin{subfigure}{0.48\textwidth}
                        \centering
                        \begin{tikzpicture}[node distance = \Ysep and \Xsep]
                            \node[base] (0) {0};
                            \node[base, above = of 0] (1) {1};
                            \node[d-node, right = of 0] (2) {2};
                            \node[d-node, right = of 2] (3) {3};
                            \node[d-node, above = of 3] (4) {4};

                            \path[uedge]
                            (0) edge (2)
                            (0) edge [bend right] (3)
                            (0) edge (4)
                            (1) edge (2)
                            (1) edge (3)
                            (1) edge (4)
                            (2) edge (3)
                            (2) edge (4)
                            (3) edge (4);
                        \end{tikzpicture}
                        \caption{$ \ell=0 $.}
                    \end{subfigure}
                    \hfill
                    \begin{subfigure}{0.48\textwidth}
                        \centering
                        \begin{tikzpicture}[node distance = \Ysep and \Xsep]
                            \node[base] (0) {0};
                            \node[base, above = of 0] (1) {1};
                            \node[base, right = of 0] (2) {2};
                            \node[base, right = of 2] (3) {3};
                            \node[base, above = of 3] (4) {4};

                            \path[] (0) -- (1) node[midway]  (V') {\textbf{V'}};
                            \node[fitbox, fit=(0) (1) (V')] (box) {};

                            \path[uedge]
                            (0) edge (2)
                            (0) edge [bend right] (3)
                            (0) edge (4)
                            (1) edge (2)
                            (1) edge (3)
                            (1) edge (4)
                            (2) edge (3)
                            (2) edge (4)
                            (3) edge (4);
                        \end{tikzpicture}
                        \caption{$ \ell=1 $.}
                    \end{subfigure}
                \end{figure}
            \item
                Second prefix node set expansion ($ S = \lbrace 0,
                1 \rbrace, \mathcal{S} = \left[\lbrace 0, 1 \rbrace\right] $).

                For $ \ell = 0 $, no new independencies are found.
                For $\ell = 1$, we remove several edges based on conditional
                independencies found.
                For instance, the edge $(0,3)$ is removed because $0
                \CI 3 \mid \left(S \setminus \{0\}\right) \cup \{2\}$.
                All other non-adjacencies in the true graph are also
                discovered and their corresponding edges removed
                during this phase.
                Additionally, we find $0 \not\CI 3 \mid S \setminus
                \{0\}$ and $0 \CI 3 \mid \left(S \setminus
                \{0\}\right) \cup \{2\}$, which identifies $3 \in
                \meekoneset^1$.
                Similarly, $1 \CI 3 \mid \left(S \setminus
                \{1\}\right) \cup \{2\}$ and $1 \not\CI 3
                \mid \left(S \setminus \{1\}\right) \cup \{2,4\}$
                identifies $4 \in \vstructset^1$.
                The working set is updated to $ V' = V' \setminus
                (\meekoneset^1 \cup \vstructset^1) = \{2,3,4\}
                \setminus (\{3\} \cup \{4\}) = \{2\} $.
                Finally, since $ V' $ contains no clique of size $ 2
                $, this expansion stops here.

                \begin{figure}[H]
                    \centering
                    \begin{subfigure}{0.48\textwidth}
                        \centering
                        \begin{tikzpicture}[node distance = \Ysep and \Xsep]
                            \node[base] (0) {0};
                            \node[base, above = of 0] (1) {1};
                            \node[base, right = of 0] (2) {2};
                            \node[base, right = of 2] (3) {3};
                            \node[base, above = of 3] (4) {4};

                            \path[] (0) -- (1) node[midway]  (S) {\textbf{S}};
                            \node[fitbox, fit=(0) (1) (S)] (box) {};

                            \path[uedge]
                            (0) edge (2)
                            (0) edge [bend right] (3)
                            (0) edge (4)
                            (1) edge (2)
                            (1) edge (3)
                            (1) edge (4)
                            (2) edge (3)
                            (2) edge (4)
                            (3) edge (4);
                        \end{tikzpicture}
                        \caption{$ \ell=0 $.}
                    \end{subfigure}
                    \hfill
                    \begin{subfigure}{0.48\textwidth}
                        \centering
                        \begin{tikzpicture}[node distance = \Ysep and \Xsep]
                            \node[base] (0) {0};
                            \node[base, above = of 0] (1) {1};
                            \node[base, right = of 0] (2) {2};
                            \node[f-node, right = of 2] (3) {3};
                            \node[d-node, above = of 3] (4) {4};

                            \path[] (0) -- (1) node[midway]  (S) {\textbf{S}};
                            \node[fitbox, fit=(0) (1) (S)] (box) {};

                            \path[uedge]
                            (0) edge (2)
                            (1) edge (2)
                            (1) edge (4)
                            (2) edge (3)
                            (3) edge (4);

                            \path[uedge, transparent] (0) edge [bend right] (3);
                        \end{tikzpicture}
                        \caption{$ \ell=1 $.}
                    \end{subfigure}
                \end{figure}
            \item
                Third prefix node set expansion ($ S = \lbrace 0,
                    1, 2 \rbrace, \mathcal{S} = \left[\lbrace 0, 1
                \rbrace, \lbrace 2 \rbrace\right] $).

                We find $2 \not\CI 4 \mid S \setminus \{2\}$ and $2 \CI 4 \mid
                \left(S \setminus \{2\}\right) \cup \{3\}$, so $4 \in
                \meekoneset^1$.

                \begin{figure}[H]
                    \centering
                    \begin{subfigure}{0.48\textwidth}
                        \centering
                        \begin{tikzpicture}[node distance = \Ysep and \Xsep]
                            \node[base] (0) {0};
                            \node[base, above = of 0] (1) {1};
                            \node[base, right = of 0] (2) {2};
                            \node[base, right = of 2] (3) {3};
                            \node[base, above = of 3] (4) {4};

                            \path[] (0) -- (1) node[midway]  (S) {\textbf{S}};
                            \node[fitbox, fit=(0) (1) (2) (S)] (box) {};

                            \path[uedge]
                            (0) edge (2)
                            (1) edge (2)
                            (1) edge (4)
                            (2) edge (3)
                            (3) edge (4);

                            \path[uedge, transparent] (0) edge [bend right] (3);
                        \end{tikzpicture}
                        \caption{$ \ell=0 $.}
                    \end{subfigure}
                    \hfill
                    \begin{subfigure}{0.48\textwidth}
                        \centering
                        \begin{tikzpicture}[node distance = \Ysep and \Xsep]
                            \node[base] (0) {0};
                            \node[base, above = of 0] (1) {1};
                            \node[base, right = of 0] (2) {2};
                            \node[base, right = of 2] (3) {3};
                            \node[f-node, above = of 3] (4) {4};

                            \path[] (0) -- (1) node[midway]  (S) {\textbf{S}};
                            \node[fitbox, fit=(0) (1) (2) (S)] (box) {};

                            \path[uedge]
                            (0) edge (2)
                            (1) edge (2)
                            (1) edge (4)
                            (2) edge (3)
                            (3) edge (4);

                            \path[uedge, transparent] (0) edge [bend right] (3);
                        \end{tikzpicture}
                        \caption{$ \ell=1 $.}
                    \end{subfigure}
                \end{figure}
            \item
                Fourth prefix node set expansion ($ S = \lbrace 0,
                    1, 2, 3 \rbrace, \mathcal{S} = \left[\lbrace 0, 1
                \rbrace, \lbrace 2 \rbrace, \lbrace 3 \rbrace\right] $).

                Only one node remains outside of $S$.
                No adjacencies are removed.

                \begin{figure}[H]
                    \centering
                    \begin{subfigure}{0.48\textwidth}
                        \centering
                        \begin{tikzpicture}[node distance = \Ysep and \Xsep]
                            \node[base] (0) {0};
                            \node[base, above = of 0] (1) {1};
                            \node[base, right = of 0] (2) {2};
                            \node[base, right = of 2] (3) {3};
                            \node[base, above = of 3] (4) {4};

                            \path[] (0) -- (1) node[midway]  (S) {\textbf{S}};
                            \def\mypadding{0.3cm}
                            \draw[fitbox] ($ (1.north west) +
                            (-\mypadding, \mypadding) $)
                            -- ($ (1.north -| 2)!.5!(1.north -| 3) +
                            (0, \mypadding / 2) $)
                            -- ($ (4.south -| 2)!.5!(4.south -| 3) +
                            (0, -\mypadding) $)
                            -- ($ (4.south -| 3.east) + (\mypadding /
                            2, -\mypadding) $)
                            -- ($ (3.south east) + (\mypadding, -\mypadding) $)
                            -- ($ (0.south west) + (-\mypadding, -\mypadding) $)
                            -- cycle;

                            \path[uedge]
                            (0) edge (2)
                            (1) edge (2)
                            (1) edge (4)
                            (2) edge (3)
                            (3) edge (4);

                            \path[uedge, transparent] (0) edge [bend right] (3);
                        \end{tikzpicture}
                        \caption{$ \ell=0 $.}
                    \end{subfigure}
                    \hfill
                    \begin{subfigure}{0.48\textwidth}
                        \centering
                        \begin{tikzpicture}[node distance = \Ysep and \Xsep]
                            \node[base] (0) {0};
                            \node[base, above = of 0] (1) {1};
                            \node[base, right = of 0] (2) {2};
                            \node[base, right = of 2] (3) {3};
                            \node[base, above = of 3] (4) {4};

                            \path[] (0) -- (1) node[midway]  (S) {\textbf{S}};
                            \def\mypadding{0.3cm}
                            \draw[fitbox] ($ (1.north west) +
                            (-\mypadding, \mypadding) $)
                            -- ($ (1.north -| 2)!.5!(1.north -| 3) +
                            (0, \mypadding / 2) $)
                            -- ($ (4.south -| 2)!.5!(4.south -| 3) +
                            (0, -\mypadding) $)
                            -- ($ (4.south -| 3.east) + (\mypadding /
                            2, -\mypadding) $)
                            -- ($ (3.south east) + (\mypadding, -\mypadding) $)
                            -- ($ (0.south west) + (-\mypadding, -\mypadding) $)
                            -- cycle;

                            \path[uedge]
                            (0) edge (2)
                            (1) edge (2)
                            (1) edge (4)
                            (2) edge (3)
                            (3) edge (4);

                            \path[uedge, transparent] (0) edge [bend right] (3);
                        \end{tikzpicture}
                        \caption{$ \ell=1 $.}
                    \end{subfigure}
                \end{figure}

            \item
                End of prefix node set expansion ($ S = \lbrace 0,
                    1, 2, 3, 4 \rbrace, \mathcal{S} = \left[\lbrace 0, 1
                        \rbrace, \lbrace 2 \rbrace, \lbrace 3 \rbrace,
                \lbrace 4 \rbrace\right] $).
                \begin{figure}[H]
                    \centering
                    \begin{subfigure}{0.48\textwidth}
                        \centering
                        \begin{tikzpicture}[node distance = \Ysep and \Xsep]
                            \node[base] (0) {0};
                            \node[base, above = of 0] (1) {1};
                            \node[base, right = of 0] (2) {2};
                            \node[base, right = of 2] (3) {3};
                            \node[base, above = of 3] (4) {4};

                            \path[] (0) -- (1) node[midway]  (S) {\textbf{S}};
                            \node[fitbox, fit=(0) (1) (2) (3) (4) (S)] (box) {};

                            \path[uedge]
                            (0) edge (2)
                            (1) edge (2)
                            (1) edge (4)
                            (2) edge (3)
                            (3) edge (4);

                            \path[uedge, transparent] (0) edge [bend right] (3);
                        \end{tikzpicture}
                    \end{subfigure}
                \end{figure}
        \end{itemize}
    \end{step}

    \begin{step}[Edge replacement, Lines
        \ref{algoline:final-loop-start}--\ref{algoline:final-loop-end}]
        Since all remaining edges connect nodes in different
        components of the partition $\mathcal{S}$, we orient them
        from components with smaller indices to those with larger indices.
        \begin{figure}[H]
            \centering
            \begin{tikzpicture}[node distance = \Ysep and \Xsep]
                \node[base] (0) {0};
                \node[base, above = of 0] (1) {1};
                \node[base, right = of 0] (2) {2};
                \node[base, right = of 2] (3) {3};
                \node[base, above = of 3] (4) {4};

                \path[] (0) -- (1) node[midway]  (S1) {$ S_1 $};
                \node[fitbox, fit=(0) (1) (S1)] (box) {};

                \node[above = 0.1cm of 2] (S2) {$ S_2 $};
                \node[fitbox, fit=(2) (S2)] (box) {};

                \node[above left = -0.15cm of 3] (S3) {$ S_3 $};
                \node[fitbox, fit=(3) (S3)] (box) {};

                \node[below left = -0.15cm of 4] (S4) {$ S_4 $};
                \node[fitbox, fit=(4) (S4)] (box) {};

                \path[dedge]
                (0) edge (2)
                (1) edge (2)
                (1) edge (4)
                (2) edge (3)
                (3) edge (4);

            \end{tikzpicture}
        \end{figure}
    \end{step}

    \begin{step}[End, Line~\ref{algoline:end}]
        The essential graph is now fully recovered.
        \begin{figure}[h!]
            \centering
            \begin{tikzpicture}[node distance = \Ysep and \Xsep]
                \node[base] (0) {0};
                \node[base, above = of 0] (1) {1};
                \node[base, right = of 0] (2) {2};
                \node[base, right = of 2] (3) {3};
                \node[base, above = of 3] (4) {4};

                \path[dedge]
                (0) edge (2)
                (1) edge (2)
                (1) edge (4)
                (2) edge (3)
                (3) edge (4);
            \end{tikzpicture}
        \end{figure}
    \end{step}
\end{example}

 \section{Example of correctness without faithfulness}\label{app:consistency}

The following example serves to prove Proposition~\ref{prop:algo-assumptions}.

\begin{example}
    Suppose we have a probability distribution $ \mathbb{P} $
    modeled by the graph $ \mathcal{G} $ in
    Figure~\ref{fig:example-algo-learning} and such that the random
    variables $ X_i $ are related to each other by the following
    structural equations
    \begin{equation}
        X_j = \sum_{i < j} a_{ij} X_i + \varepsilon_j
    \end{equation}
    where $ \varepsilon_j \sim \mathcal{N}(0,1)$ and $ a_{ij} \in [-1,1] $ are
    parameters with $ a_{ij} \neq 0 $ if and only if $ X_i$ causes $ X_j $.

    \begin{figure}[htpb]
        \centering
        \newcommand{\labeledge}[3]{
    \draw[dedge] (#1) -- node[midway, auto, font=\small] {$ a_{#1#2} $} (#2)
}

\begin{tikzpicture}
    \node[base] (0) at (0,0) {$ X_0 $};
    \node[base] (1) at (0,2) {$ X_1 $};
    \node[base] (2) at (6,1) {$ X_2 $};
    \node[base] (3) at (3,1) {$ X_3 $};

    \labeledge{0}{1}{0.5};
    \labeledge{0}{3}{0.5};

    \labeledge{1}{3}{0.5};

    \labeledge{2}{3}{0.5};
\end{tikzpicture}
         \caption{
            An example demonstrating how Algorithm~\ref{alg:learning}
            can recover the correct causal graph when the generating
            probability distribution is Markov but not faithful to that graph.
        }
        \label{fig:example-algo-learning}
    \end{figure}
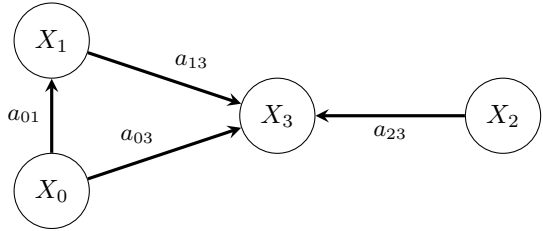

    According to Proposition~4.1.\ in \cite{uhler2013geometry}, the condition:
    \begin{equation}\label{eq:faithfulness-violation}
        a_{03}a_{13} - a_{01} = 0
    \end{equation}
    holds if and only if the conditional independence $ X_0 \CI X_1
    \mid \{X_2, X_3 \} $ is induced by $ \mathbb{P} $.
    Therefore, if Equation~\eqref{eq:faithfulness-violation} is
    satisfied, this constitutes a faithfulness violation.

    By selecting $ a_{ij} $ such that this is the only faithfulness
    violation present, Algorithm~\ref{alg:learning} successfully
    recovers the true essential graph $ \mathcal{E}(\mathcal{G}) $.
    This is due to the algorithm identifying $X_3$ as a downstream
    node very early.
    Specifically, for $S = \varnothing$, the algorithm finds that
    $X_1$ and $X_2$ are independent, but become dependent when
    conditioning on $X_3$.
    This reveals the v-structure $X_1 \rightarrow X_3 \leftarrow X_2$
    and places $X_3$ in $\vstructset^0$.
    Because $X_3$ is identified as downstream, it is excluded from
    certain future conditioning sets, meaning the test ${X_0 \CI X_1
    \mid \{X_2, X_3 \}}$---which represents the only faithfulness
    violation---is never performed by Algorithm~\ref{alg:learning}.
\end{example}
 \section{Implementation details}\label{app:code}

All graph manipulations required by Algorithm~\ref{alg:learning}
are implemented using the NetworkX {P}ython library
(\citealp{networkx2008}), available under the 3-clause BSD license.

We now provide the specific implementations of the steps in Algorithm
\ref{alg:learning}.

\begin{algorithm}
    \caption{Algorithm used to remove edges from $ \mathcal{E} $
    in Algorithm \ref{alg:learning}}
    \label{alg:remove-edges}
    \begin{algorithmic}[1]
        \Require CI queries from $ \mathbb{P} $, graph
        $ \mathcal{E} $, prefix set $ S $, nodes $ V' $, condition set
        size $ m $.
        \Ensure $ \mathcal{E} $, sepset.

        \For{$ u\in V, v \in V' $ such that $ u - v $ in $ \mathcal{E} $}
        \Repeat
        \State Choose a new set $ W \subseteq V' \setminus \{u,v\} $
        with $ |W| = m $;
        \If{$ u \CI v \mid S \cup W $}
        \State Remove $ u - v $ from $ \mathcal{E} $;
        \State $ \text{sepset}({u,v}) \gets W $;
        \EndIf
        \Until{$ u $ and $ v $ are no longer adjacent or all sets $ W
        $ have been considered}
        \EndFor
        \State \Return $ \mathcal{E} $, sepset
    \end{algorithmic}
\end{algorithm}

\begin{algorithm}
    \caption{Algorithm used to calculate $ \vstructset^m $}
    \label{alg:vstructset}
    \begin{algorithmic}[1]
        \Require CI queries from $ \mathbb{P} $, graph
        $ \mathcal{E} $, prefix set $ S $, nodes $ V' $, sepset from
        Algorithm \ref{alg:remove-edges}.
        \Ensure $\vstructset^m$.

        \State $ \vstructset^m \gets \varnothing $;
        \State $ P \gets \varnothing $;

        \For{$ u \in V $, $ v \in V' $ such that $ |\text{sepset}(u,v)
        \setminus S| = m $} \Comment{Identify v-structures}
        \For{$ w $ in $ \text{adj}(\mathcal{E}, u) \cap
        \text{adj}(\mathcal{E}, v) \cap V' $}
        \If{$ w $ not in $ \text{sepset(u, v)} $}
        \State $ \vstructset^m \gets \vstructset^m \cup \{w\} $;
        \State $ P \gets P \cup \{\{u, v\}\} $;
        \EndIf
        \EndFor
        \EndFor

        \For{$ \{u, v\} $ in $ P $} \Comment{Identify descendants of
        the v-structures}
        \For{$ w $ in $ V' \setminus \vstructset^m $}
        \If{$ u \not\CI v \mid S \cup \text{sepset}(u, v) \cup \{w\} $}
        \State $ \vstructset^m \gets \vstructset^m \cup \{w\} $;
        \EndIf
        \EndFor
        \EndFor

        \State \Return $ \vstructset^m $
    \end{algorithmic}
\end{algorithm}

\begin{algorithm}
    \caption{Algorithm used to calculate $ \meekoneset^m $}
    \label{alg:meekoneset}
    \begin{algorithmic}[1]
        \Require CI queries from $ \mathbb{P} $, prefix set $S$,
        nodes $V'$, sepset from Algorithm \ref{alg:remove-edges}.
        \Ensure $\meekoneset^m$.

        \State $ \meekoneset^m \gets \varnothing $;

        \For{$ u \in S $, $ v \in V' $ such that $\text{sepset}(u,v) > 0 $}\label{fset:imp}
        \Repeat
        \State Choose a new set $ W \subseteq V' \setminus \{u,v\} $
        with $ |W| = m $;
        \If{$ u \CI v \mid S \cup W $}
        \State $ \meekoneset^m \gets \meekoneset^m \cup \{v\} $;
        \EndIf
        \Until{$ v \in \meekoneset^m $or all sets $ W $ have been considered}
        \EndFor

        \State \Return $ \meekoneset^m $
    \end{algorithmic}
\end{algorithm}

To solve the k-clique problem, we adapt the Bron-Kerbosch algorithm
\citep{bronkerbosch}, originally used for enumerating maximal cliques,
by terminating the search as soon as a clique of size $ k $ is found.

To efficiently compute the sets $ \vstructset^m $ and $ \meekoneset^m
$, we optimize the process by removing adjacencies once a separating
set is found (see Algorithm~\ref{alg:remove-edges}) and storing these sets.

The calculation of $ \vstructset^m $ (detailed in
Algorithm~\ref{alg:vstructset}) is a two-stage process.
First, it identifies initial v-structures by iterating through non-adjacent
node pairs $(u,v)$ and checking if a common neighbor $w$ is absent from
their separating set, reusing the results from Algorithm~\ref{alg:remove-edges}.
Second, it finds all descendants of these v-structures by performing
targeted additional CI tests to see which nodes $w$, when added to a
conditioning set, break a known independence.
Lemma~\ref{lem:first-in-vstructset} guarantees that that the set
$\vstructset^m$ consists exclusively of colliders within v-structures
and their descendants.

The calculation of $ \meekoneset^m $ (Algorithm \ref{alg:meekoneset})
identifies the desired nodes by leveraging the previously identified
separating sets.
The definition for $w \in \meekoneset^m$ requires two conditions: (1)
$u \not\CI w \mid S$ for some $u \in S$, and (2) $u \CI w \mid S \cup
W$ for some $W \subseteq \bar{S}$ of size $m$.
The first condition is implicitly satisfied by Line~\ref{fset:imp}.
Therefore, simply finding a separating set $W$ of size $m$ for a pair
$(u,w)$ with $u \in S$ is sufficient to place $w$ in $\meekoneset^m$.

\section{Algorithm variant with additional tests}\label{app:variant}

We introduce a variant of Algorithm \ref{alg:learning}, which we call
\texttt{\algoshortname+}, that replaces the final orientation loop
(Lines~\ref{algoline:final-loop-start}--\ref{algoline:final-loop-end})
with additional CI tests (see Algorithm~\ref{alg:additional-tests}).
This variant returns a different graph only if faithfulness
is violated.
In our experiments (Figure \ref{fig:neighborhood}), this variant acts
as a normalized baseline for comparison.

\begin{algorithm}
    \caption{Modification of the final loop in Algorithm
        \ref{alg:learning}
    (Lines~\ref{algoline:final-loop-start}--\ref{algoline:final-loop-end})}
    \label{alg:additional-tests}
    \begin{algorithmic}[1]
        \State Create empty graph $\mathcal{D}$ on the node set $V$.
        \For{$S_i$ in $\mathcal{S}=[S_1,\dots,S_m]$}
        \State Add $v-w$ to $\mathcal{D}$ iff $v\not\CI w\mid
        S_1 \cup \dots
        \cup S_i \setminus \{v,w\}$.
        \EndFor

        \For{$v \in S_i, w \in S_j$ with $i<j$}
        \State Add $v\to w$ to $\mathcal{D}$ iff $v\not\CI w\mid
        S_{1}\cup\dots\cup S_{j} \setminus \{v,w\}$.
        \EndFor

        \State \Return $ \mathcal{D} $
    \end{algorithmic}
\end{algorithm}

Next, we will show that under faithfulness, $ \mathcal{D} $ is
precisely $ \mathcal{E}(\mathcal{G}) $.
We will start by proving that if we have $ v - w $ in $ \mathcal{D} $
then $ v $ and $ w $ will be adjacent in $ \mathcal{G} $.
If we have $ v - w $ in $ \mathcal{D} $, we must have $ v, w \in
S_i $ for some $ S_i \in \mathcal{S} $.
Denote $ S = (S_1 \cup \dots \cup S_i) \setminus \{v,w\}$.
Since $ v \not\CI w \mid S$ there exists an active
path between $ v $ and $ w $ given $ S$.
Assume this path contains more than two nodes.
We have that $ S \cup \{v,w\} $ is prefix, therefore, there
must exist a collider on the path.
Any collider on the active path must also be in $ S $, however, in this
case, the adjacent nodes will also be in $ S $ and the path would
be inactive.
The only remaining case is if $ v \sim c \sim w $ is a v-structure
for some $ c \in S $ but since $ v,w \in S_i $ then $ c \in S_i $ and
by Lemma~\ref{lem:vstruct-sets} we have a contradiction.
Therefore, we have $ v \sim w $.

If we have $ v \rightarrow w $ in $ \mathcal{D} $ we must have $
v \in S_i $, $ w \in S_j $ for some $ S_i, S_j \in \mathcal{S} $
and $ i < j $.
Denote $ S = (S_1 \cup \dots \cup S_j) \setminus \{v,w\} $.
Since $ v \not\CI w \mid S$ there exists an active
path between $ v $ and $ w $ given $ S $.
We have that $ S \cup \{v,w\} $ is prefix, therefore, there must
exist a collider on the path.
Any collider in the active path must be in $ S $, however, in this
case, the adjacent nodes will also be in $ S $ and the path would
be inactive.
The only remaining case is if $ v \rightarrow c \leftarrow w $, but
that means $ c \in S_j $.
If not $ v \sim w $ then $ v \sim c \sim w $ is a v-structure, by
Lemma \ref{lem:vstruct-sets} we have $ c \in S_k $ with $ k > j $, a
contradiction.
Since $ S_1 \cup \dots \cup S_i $ is prefix we must have $ v
\rightarrow w $.

By Lemma \ref{lem:directed-edges} there are no directed edges in
the essential graph of $ \mathcal{G} $ other than those in $
\mathcal{D} $.

Finally, we will prove that if $ v \sim w $ in $ \mathcal{G} $ then $
v \sim w $ in $ \mathcal{D} $ and therefore we have all the
adjacencies.
If $ v \sim w $ in $ \mathcal{G} $ we will have $ v \not\CI w \mid W
$ for any $ W \subset V $.
With the intra and inter component tests we will also have $ v \sim w
$ in $ \mathcal{D} $.
 \section{Experimental setup}\label{app:experiments}

\textbf{Hardware.} All experiments were conducted on a single device
with 1TB of RAM.

The conditional independence test employed across all compared
algorithms was the partial correlation test (using Fisher's
z-transform) at a significance level of $ \alpha = 0.05 $, unless
otherwise specified.
However, to ensure fair comparisons, we utilized different
implementations of this test: results presented in
Figures~\ref{fig:neighborhood} are using the
implementation from the \texttt{causaldag} package
(\citealp{squires2018causaldag}), while results in Figures
\ref{fig:sergio},~\ref{fig:airfoil}~and~\ref{fig:cl-neighborhood} are
using the implementation provided by the \texttt{causal-learn}
package (\citealp{zheng2024causal}).
In both cases, the default implementation of each tester and algorithm is used.

\subsection{Linear Gaussian synthetic data}\label{app:scm-details}

In these experiments, we simulate an observational setting.
For each run, 10,000  data samples were generated from a linear
Structural Causal Model (\citealp{pearl2009causality}) with additive
Gaussian noise, defined by a randomly generated underlying causal DAG
$ \mathcal{G} $.
Edge weights for the SCM are drawn from $ [-1, -0.25] \cup [0.25, 1]
$, ensuring they are bounded away from $ 0 $.

\subsection{SERGIO}\label{app:experiments-sergio}

The SERGIO simulation framework, developed by
\citet{dibaeinia2020sergio}, enables the generation of single-cell
transcriptomics data that realistically reflects gene regulatory
networks defined by the user.

For this, we use the DS1 gene regulatory network also introduced in
\citet{dibaeinia2020sergio}, which was designed based on known
regulatory pathways in E.~coli.
This network comprises 100 genes and 258 edges.
Key structural characteristics include 10 designated regulator genes,
which are the only nodes with outgoing edges, and a maximum node
out-degree of 76.
Such a high maximum degree significantly impacts the runtime of
constraint-based algorithms like \texttt{PC} and \texttt{FCI}, which
become computationally prohibitive due to their search procedures.
Using this DS1 GRN as the ground truth, we simulated gene expression
datasets consisting of 2,700 cells.
Finally, for the conditional independence tests used by
\texttt{\algoname} and \texttt{\algoname+}, we set the significance
level to $ \alpha = 10^{-4} $.

\subsection{Real-world example}

For our real-world application, we utilize the Airfoil Self-Noise
NASA dataset, which was introduced by \citet{brooks1989airfoil}
and is available under the CC BY 4.0 license.
For this experiment, we set the significance level to $ \alpha = 10^{-4} $.

\section{Additional experiments}\label{app:additional-experiments}

This section extends our empirical evaluation of \texttt{\algoname}
and \texttt{\algoname+}.
We analyze their performance across diverse graph structures, network
sizes, and sample sizes to assess their robustness and scalability.

\subsection{Increasing neighborhood size in Erd\H{o}s-R\'enyi Graphs}

Figure~\ref{fig:cl-neighborhood} shows the execution time and
Structural {H}amming Distance.
As shown in Figure~\ref{fig:number-ci-tests}, \texttt{\algoname}
and \texttt{\algoname+} significantly reduce the number of
conditional independence tests compared to \texttt{PC}, often by more
than an order of magnitude.
Regarding accuracy, Figure~\ref{fig:edges} breaks down errors into
extra and missing edges, while Table~\ref{tab:performance-metrics} provides
complementary skeleton-level metrics, including precision, recall, and F1-score.

\begin{figure}[h!]
    \centering
    \begin{subfigure}[b]{0.48\textwidth}
        \centering
        \resizebox{\linewidth}{!}{\includegraphics{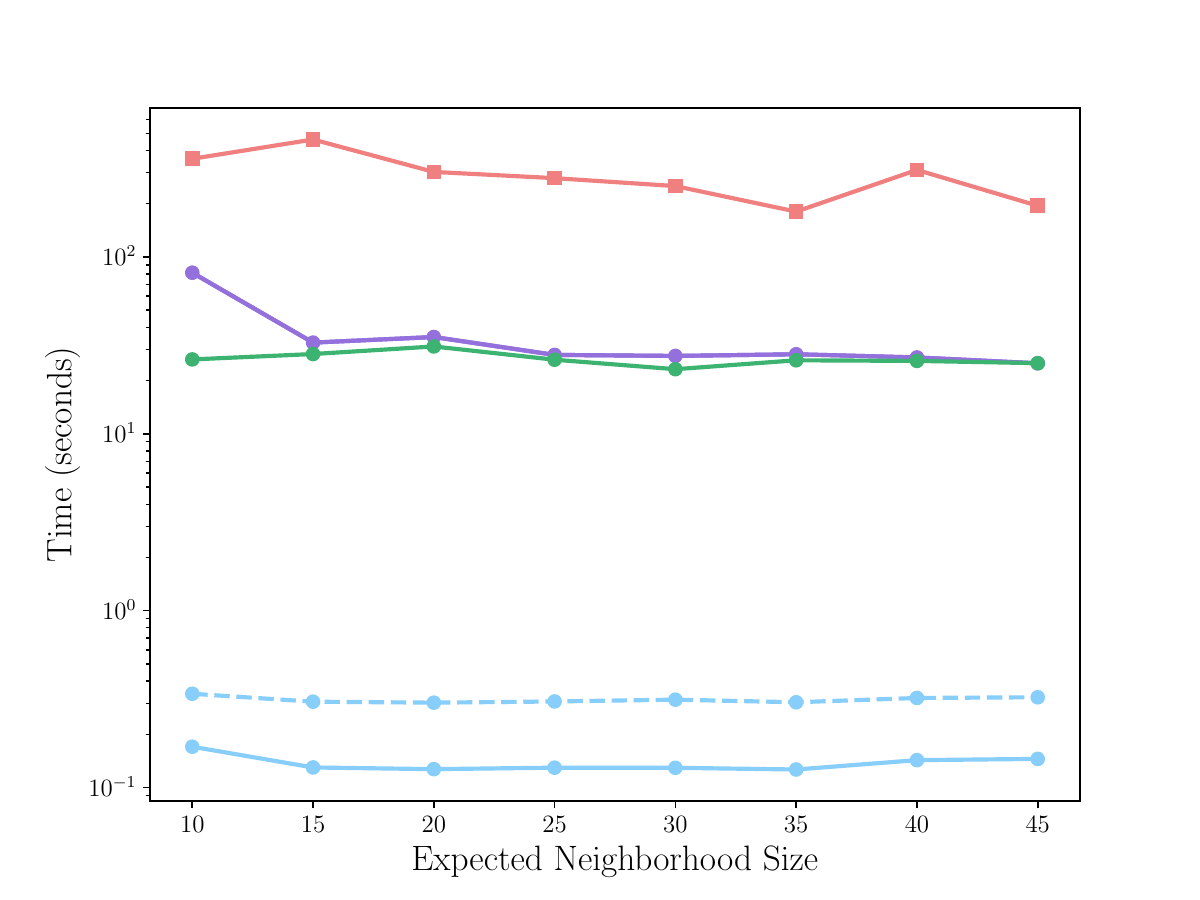}
        }
        \caption{}
    \end{subfigure}
    \hfill
    \begin{subfigure}[b]{0.48\textwidth}
        \centering
        \resizebox{\linewidth}{!}{\includegraphics{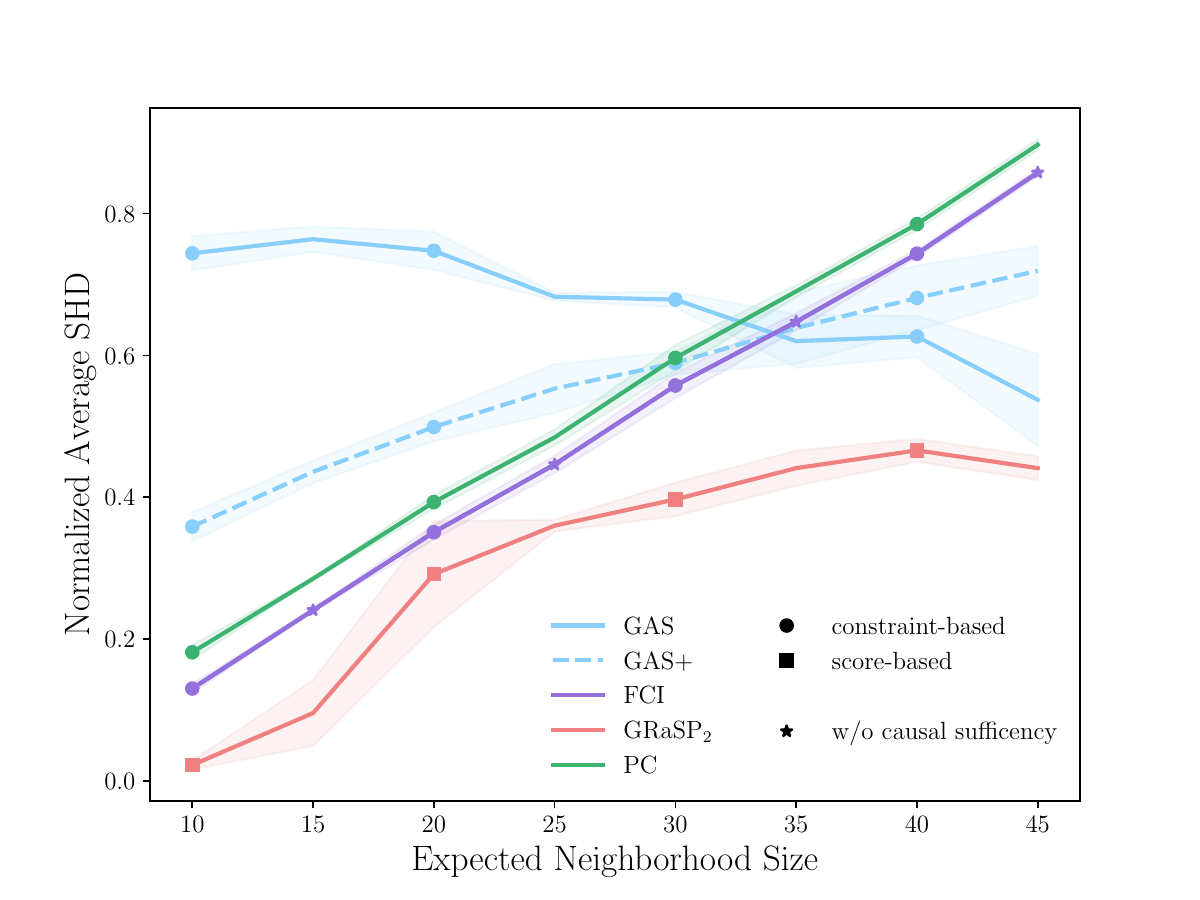}
        }
        \caption{}
        \label{fig:cl-neighborhood-accuracy}
    \end{subfigure}
    \caption{\neighborhoodcaption}
    \label{fig:cl-neighborhood}
\end{figure}

\begin{figure}
    \centering
    \resizebox{0.5\linewidth}{!}{\includegraphics{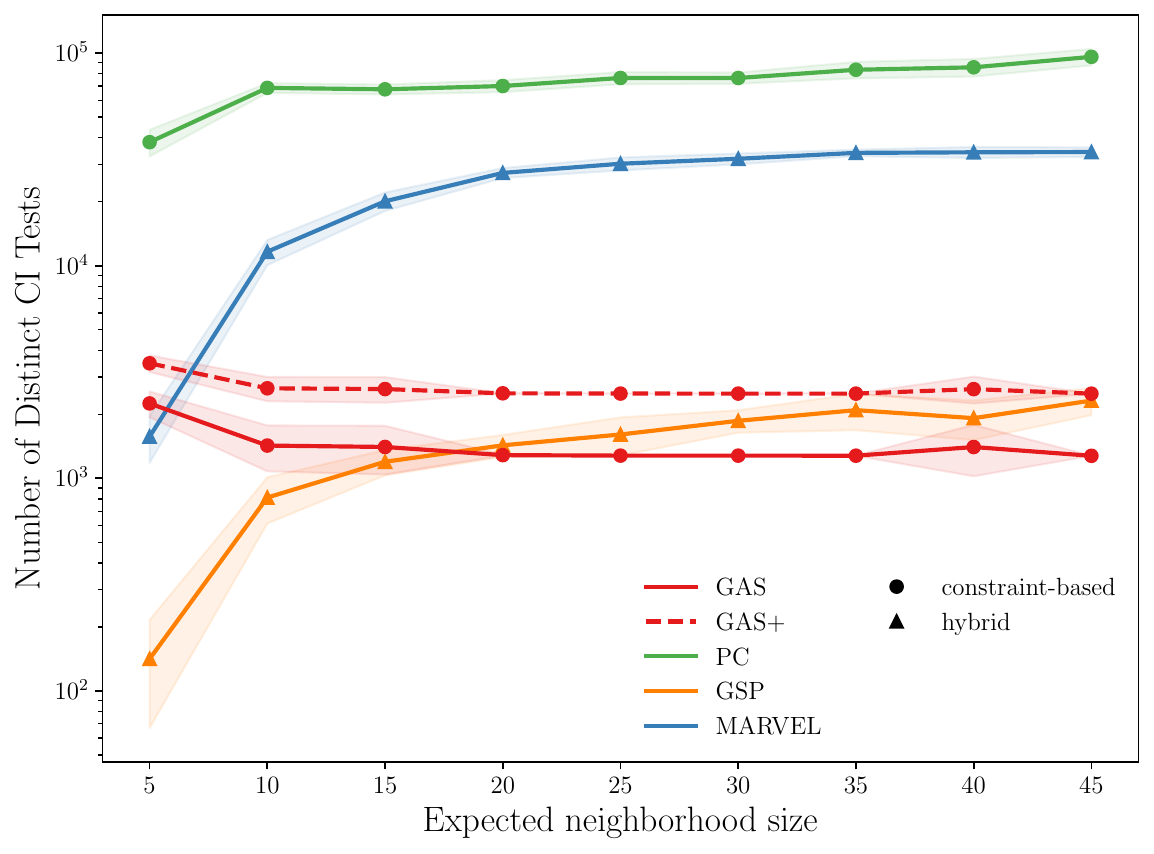}
    }
    \caption{
        Number of distinct CI tests performed across
        Erd\H{o}s-R\'enyi graphs of 50 nodes and increasing expected
        neighborhood size.
        Results are averaged over 10 runs.
    }
    \label{fig:number-ci-tests}
\end{figure}

\begin{figure}[h!]
    \centering
    \begin{subfigure}[b]{0.48\textwidth}
        \centering
        \resizebox{\linewidth}{!}{\includegraphics{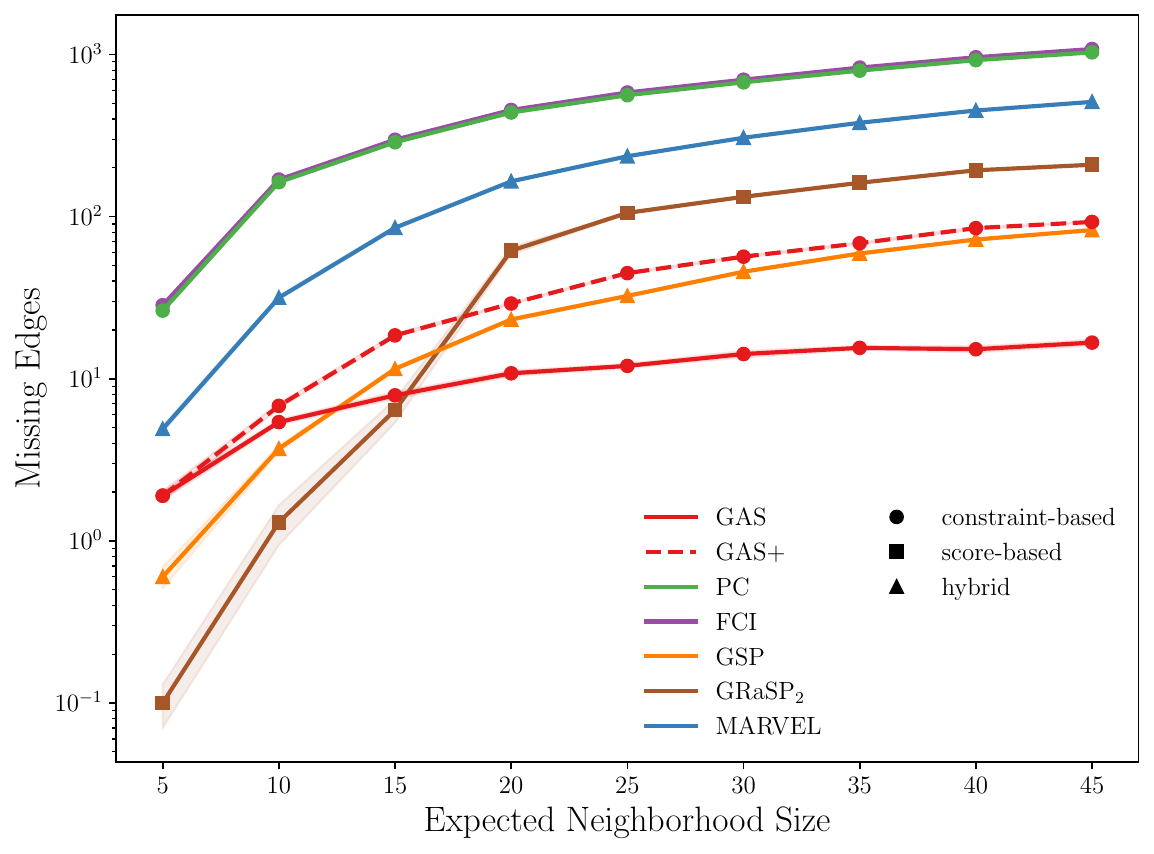}
        }
        \caption{}
    \end{subfigure}
    \hfill
    \begin{subfigure}[b]{0.48\textwidth}
        \centering
        \resizebox{\linewidth}{!}{\includegraphics{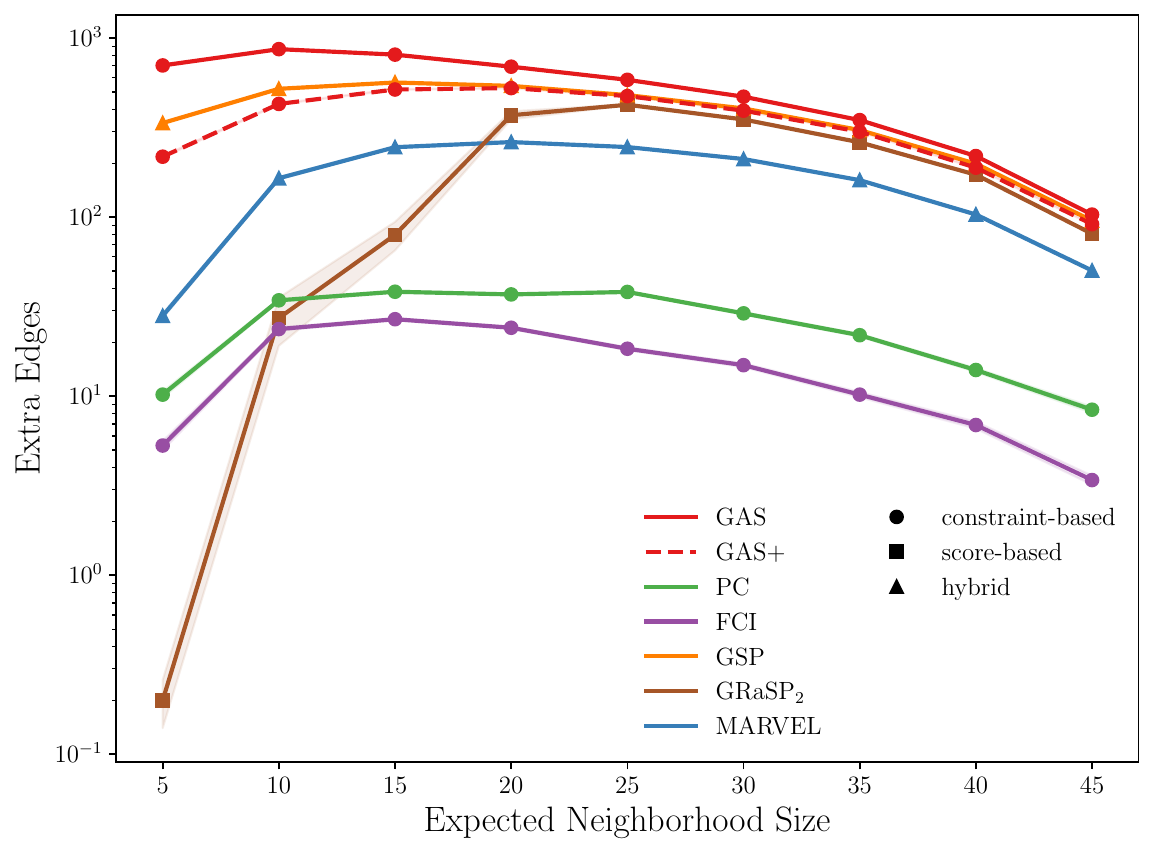}
        }
        \caption{}
    \end{subfigure}
    \caption{
        Comparison across Erd\H{o}s-R\'enyi graphs on 50 nodes
        and increasing expected neighborhood size.
        Results are averaged over 10 runs.
        (a) Number of missing edges between predicted and
        ground-truth skeletons.
        (b) Number of extra edges between predicted and
        ground-truth skeletons.
        Shaded regions show the standard deviation.
    }
    \label{fig:edges}
\end{figure}

\begin{table}[ht]
    \centering
    \caption{
        Additional skeleton metrics for \texttt{\algoname+} and
        \texttt{PC} across Erd\H{o}s-R\'enyi graphs of 50 nodes and
        increasing expected neighborhood size.
    }
    \label{tab:performance-metrics}
    \begin{tabular}{@{}ccccccc@{}}
        \toprule
        \multirow{2}{*}{\makecell{\textbf{Exp. Neigh.}\\\textbf{Size}}} &
        \multicolumn{2}{c}{\textbf{Precision}} &
        \multicolumn{2}{c}{\textbf{Recall}} &
        \multicolumn{2}{c}{\textbf{F1-Score}} \\
        \cmidrule(lr){2-3} \cmidrule(lr){4-5} \cmidrule(lr){6-7}
        & \texttt{\algoname+} & \texttt{PC} & \texttt{\algoname+} &
        \texttt{PC} & \texttt{\algoname+} & \texttt{PC} \\
        \midrule
        15 & 0.42 & \textbf{0.83} & \textbf{0.92} & 0.29 &
        \textbf{0.58} & 0.43 \\
        25 & 0.57 & \textbf{0.81} & \textbf{0.81} & 0.16 &
        \textbf{0.67} & 0.27 \\
        35 & 0.73 & \textbf{0.87} & \textbf{0.77} & 0.12 &
        \textbf{0.75} & 0.21 \\
        45 & 0.93 & \textbf{0.96} & \textbf{0.75} & 0.09 &
        \textbf{0.83} & 0.17 \\
        \bottomrule
    \end{tabular}
\end{table}

\clearpage

\subsection{Increasing number of nodes in Erd\H{o}s-R\'enyi graphs}

To evaluate scalability with respect to graph size, we conducted experiments on
Erd\H{o}s-R\'enyi graphs with an increasing number of nodes.
Figure~\ref{fig:increasing-nodes} reports the execution time and
Structural Hamming Distance (SHD), while
Figure~\ref{fig:increasing-nodes-tests} details the number of
distinct conditional independence (CI) tests. It is important to note that for the experiments with a fixed
    expected neighborhood size
    (Figure~\ref{fig:increasing-nodes-tests-fixed-neighborhood}), the
    graphs become effectively sparser as the number of nodes increases.
    \texttt{GAS} and \texttt{GAS+} demonstrate superior scaling compared
    to the standard \texttt{PC} algorithm and \texttt{MARVEL}
    \citep{mokhtarian2025recursive} when the edge probability is held
    constant (Figure~\ref{fig:increasing-nodes-tests-fixed-sparsity}).
    While the hybrid algorithm \texttt{GSP} performs the fewest distinct
    CI tests, this lower test count does not strictly correspond to lower
    computational cost, as evidenced by the superior runtime scaling of
    \texttt{GAS} and \texttt{GAS+} in Figure~\ref{fig:increasing-nodes}.

\begin{figure}[htpb]
    \centering
    \begin{subfigure}[b]{0.48\textwidth}
        \centering
        \resizebox{\linewidth}{!}{\includegraphics{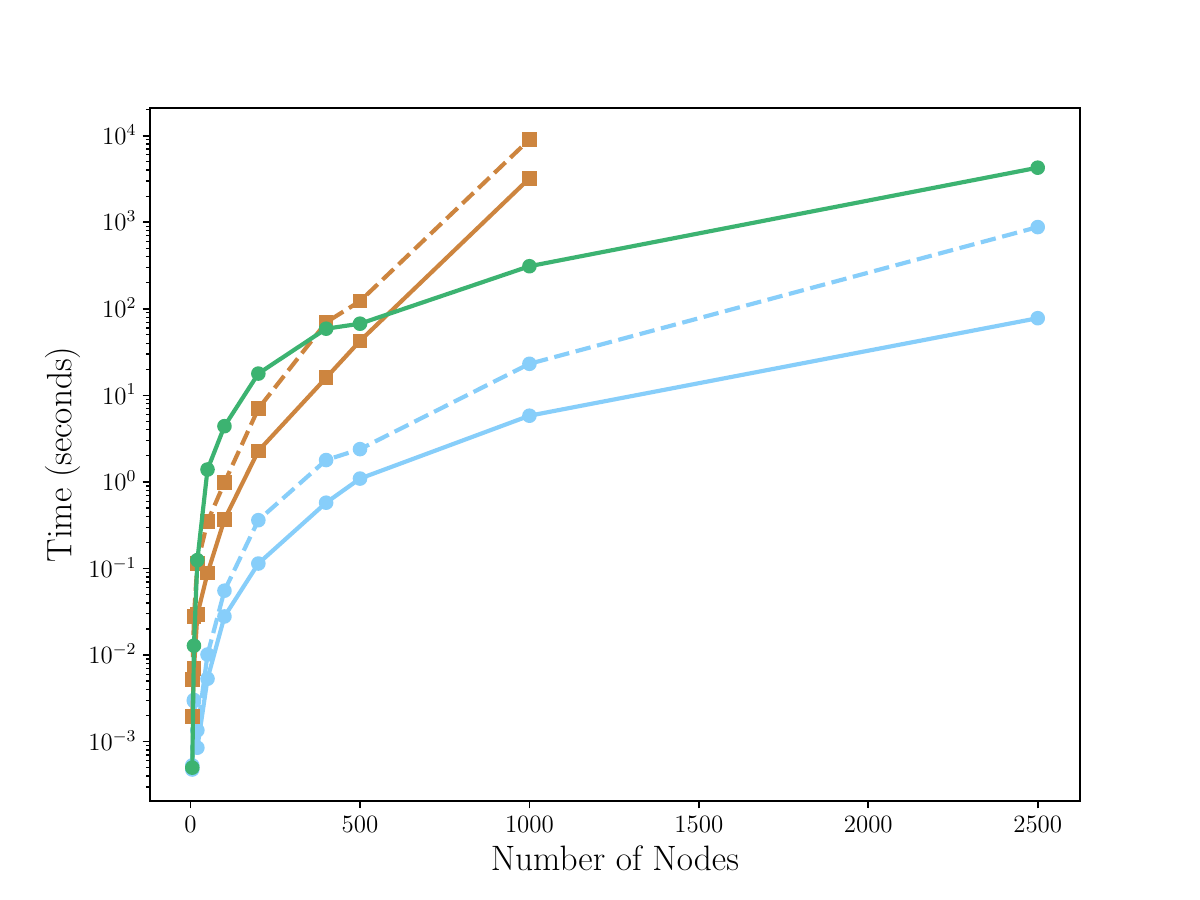}
        }
        \caption{}
    \end{subfigure}
    \hfill
    \begin{subfigure}[b]{0.48\textwidth}
        \centering
        \resizebox{\linewidth}{!}{\includegraphics{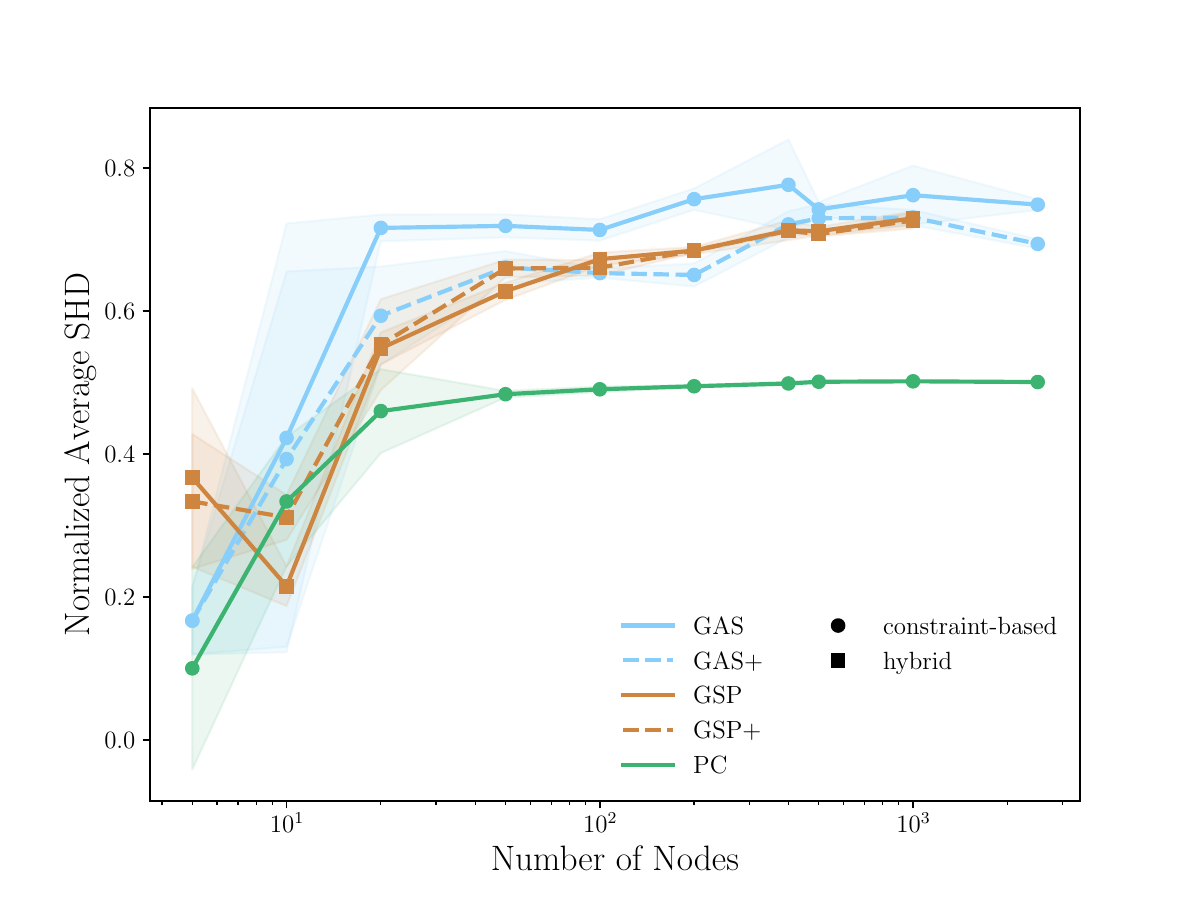}
        }
        \caption{}
    \end{subfigure}
    \hfill
    \begin{subfigure}[b]{0.48\textwidth}
        \centering
        \resizebox{\linewidth}{!}{\includegraphics{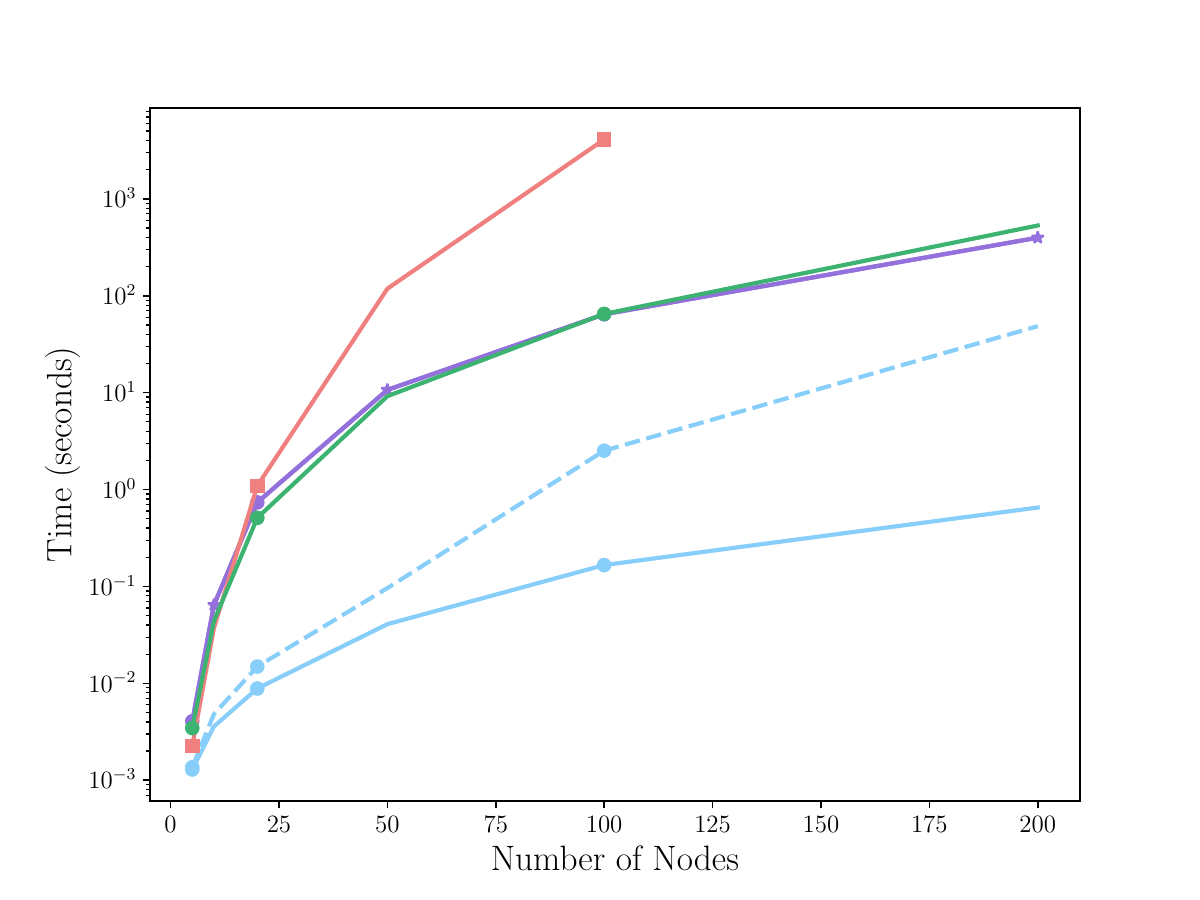}
        }
        \caption{}
    \end{subfigure}
    \hfill
    \begin{subfigure}[b]{0.48\textwidth}
        \centering
        \resizebox{\linewidth}{!}{\includegraphics{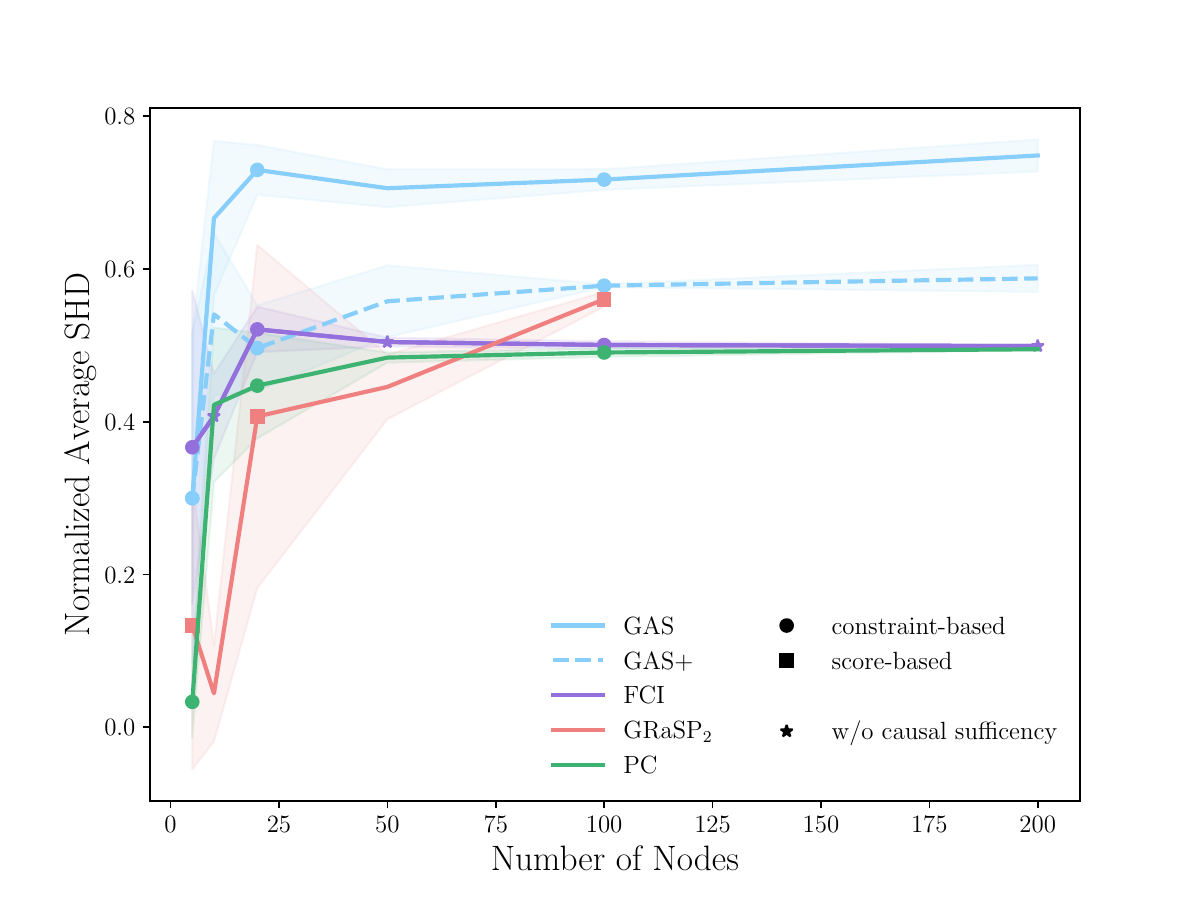}
        }
        \caption{}
    \end{subfigure}
    \caption{
        Comparison across Erd\H{o}s-R\'enyi graphs of increasing number of
        nodes and edge probability $ p=0.5 $.
        Results are averaged over 3 runs.
        (a), (c) Execution time (seconds) presented on a logarithmic scale.
        (b), (d) Accuracy comparison measured by the Structural {H}amming
        Distance between predicted
        and ground-truth graphs normalized by the number of
        possible edges.
        Shaded regions show the standard deviation.
    }
    \label{fig:increasing-nodes}
\end{figure}

\begin{figure}[htpb]
    \centering
    \begin{subfigure}[b]{0.48\textwidth}
        \centering
        \resizebox{\linewidth}{!}{\includegraphics{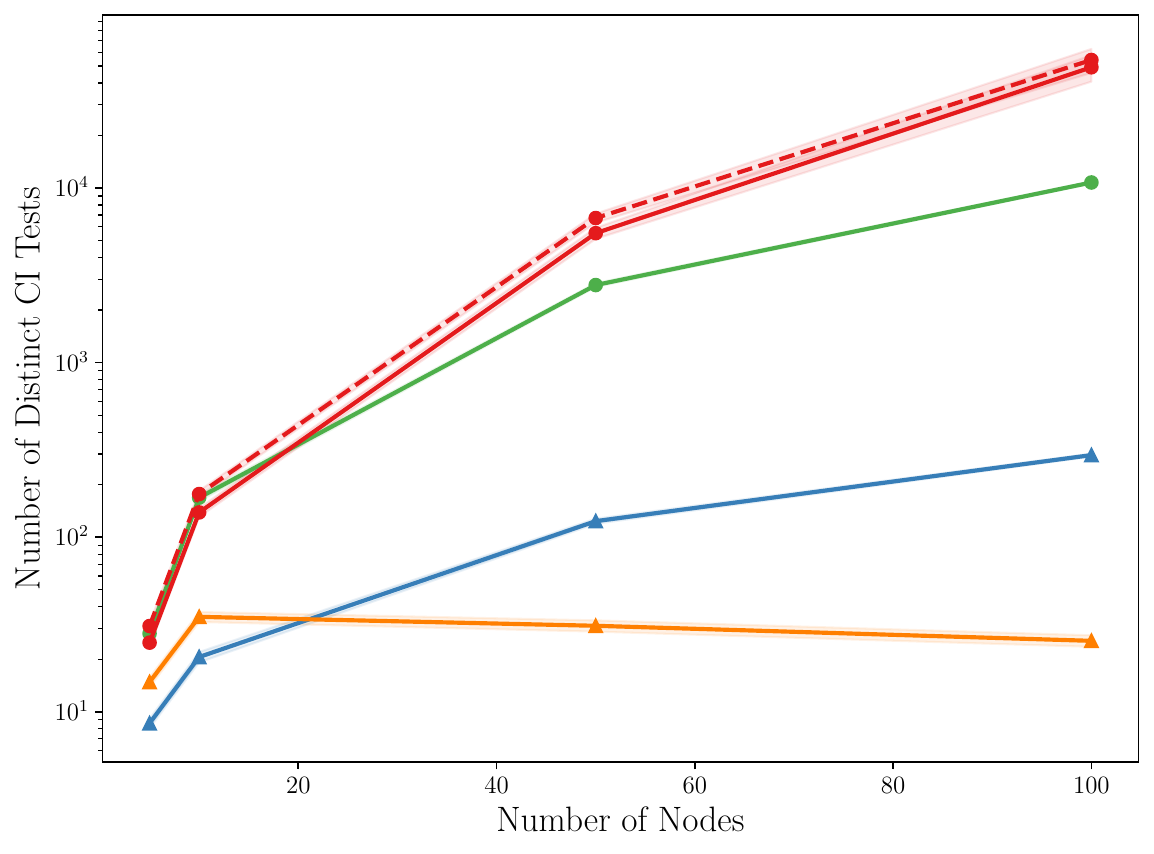}
        }
        \caption{Expected neighborhood size $2$}
        \label{fig:increasing-nodes-tests-fixed-neighborhood}
    \end{subfigure}
    \hfill
    \begin{subfigure}[b]{0.48\textwidth}
        \centering
        \resizebox{\linewidth}{!}{\includegraphics{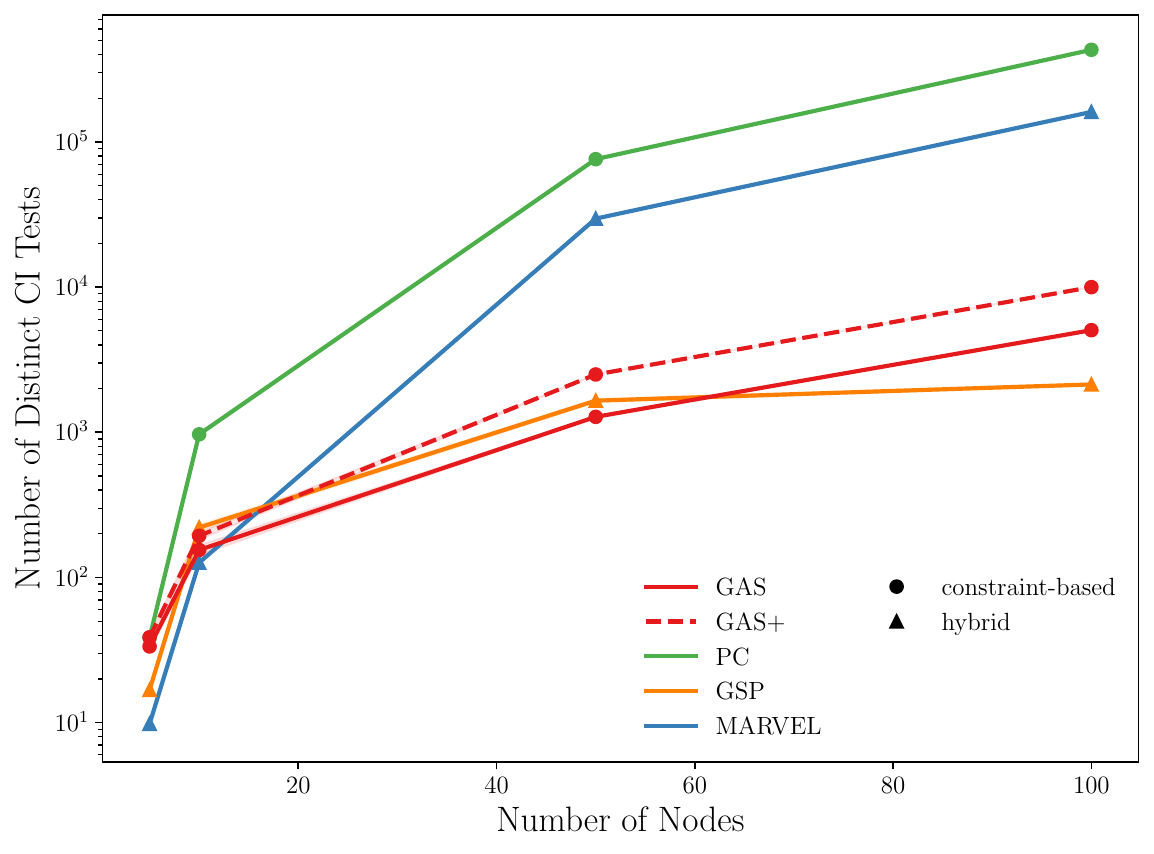}
        }
        \caption{Edge probability $p=0.5$}
        \label{fig:increasing-nodes-tests-fixed-sparsity}
    \end{subfigure}
    \caption{
        Comparison of the number of distinct conditional independence
        tests across Erd\H{o}s-R\'enyi graphs of increasing size.
        Results are averaged over 10 runs; shaded regions indicate
        the standard deviation.
    }
    \label{fig:increasing-nodes-tests}
\end{figure}

\subsection{Scale-free graphs}

Given that many real-world networks exhibit scale-free properties, we
evaluate our algorithms on graphs generated using the Barabási–Albert
(BA) model.
Table~\ref{tab:sf-graphs} reports the Structural {H}amming Distance and
execution time for graphs with increasing connectivity, controlled by
the BA model's parameter~$m$.

\begin{table}[h!]
    \centering
    \caption{
        Comparison across Barabási–Albert graphs of 40 nodes and
        increasing parameter $ m $.
        Results are averaged over 3 runs.
    }
    \label{tab:sf-graphs}
    \begin{tabular}{@{}cc ccc ccc@{}}
        \toprule
        \multirow{2}{*}{\textbf{m}} &
        \multirow{2}{*}{\makecell{\textbf{Exp.
        Neigh.}\\\textbf{Size}}} &
        \multicolumn{3}{c}{\textbf{Structural {H}amming Distance}} &
        \multicolumn{3}{c}{\textbf{Execution Time (s)}} \\
        \cmidrule(lr){3-5} \cmidrule(lr){6-8}
        & & \texttt{PC} & \texttt{\algoname} & \texttt{\algoname+}
        & \texttt{PC} & \texttt{\algoname} & \texttt{\algoname+} \\ \midrule
        5 & 8.57 & \textbf{105.00} & 380.67 & 283.33 & 1.7576 &
        \textbf{0.0377} & 0.0432 \\
        10 & 14.29 & \textbf{225.00} & 426.00 & 289.67 & 1.4815 &
        \textbf{0.0085} & 0.0138 \\
        15 & 17.14 & 297.00 & 477.33 & \textbf{296.33} & 1.7844 &
        \textbf{0.0074} & 0.0192 \\
        20 & 17.14 & 305.00 & 450.00 & \textbf{294.67} & 3.3219 &
        \textbf{0.0071} & 0.0150 \\
        \bottomrule
    \end{tabular}
\end{table}

\subsection{Varying sample size}

Finally, we investigate the impact of sample size on the accuracy of
the learned graph structures.
We generated datasets of varying sizes from a 50-node
Erd\H{o}s-R\'enyi graph with an edge probability of $p=0.5$.
Table~\ref{tab:sample-size} shows the skeleton Structural {H}amming
Distance for each method.

\begin{table}[h!]
    \centering
    \caption{
        Skeleton Structural {H}amming Distance between the
        predicted and ground-truth graphs for increasing sample
        sizes.
        Data was generated from $50$-node Erd\H{o}s-R\'enyi graphs
        with an edge probability of $p=0.5$.
        Results are averaged over 3 runs.
    }
    \label{tab:sample-size}
    \begin{tabular}{lccc}
        \toprule
        \textbf{Number of Samples} & \texttt{PC} &
        \texttt{\algoname} & \texttt{\algoname+} \\
        \midrule
        50      & \textbf{568.33} & 585.00 & 598.00 \\
        100     & \textbf{567.00} & 579.67 & 573.33 \\
        500     & 540.00 & 597.67 & \textbf{514.33} \\
        1000    & 554.33 & 596.33 & \textbf{500.33} \\
        5000    & 523.00 & 607.00 & \textbf{492.33} \\
        10000   & 538.00 & 604.00 & \textbf{476.33} \\
        100000  & 534.67 & 605.67 & \textbf{508.33} \\
        \bottomrule
    \end{tabular}
\end{table}

\end{document}